%% file: main.tex
\pdfoutput=1
\documentclass[twoside]{article}
\usepackage{color,xcolor}
\usepackage[accepted]{aistats2020_bis}
\include{param}

\usepackage[inline, shortlabels]{enumitem}  %

\usepackage[round]{natbib}

\newcommand{\Span}{\textbf{Span}}
\newcommand{\J}{\textbf{J}}
\usepackage[capitalize]{cleveref}
\crefname{thm}{Thm.}{theorems}
\crefname{fig}{Fig.}{fig.}
\crefname{prop}{Prop.}{propositions}
\crefname{cor}{Cor.}{corollaries}
\crefname{lemma}{Lem.}{lem.}
\crefname{asm}{Assumption}{assumptions}
\crefname{ex}{Example}{examples}
\crefformat{subsection}{\S#2#1#3}
\crefformat{section}{\S#2#1#3}
\crefformat{equation}{(#2\textup{#1}#3)}

\definecolor{ao}{rgb}{0.55, 0.7, 0.0}

\newif\ifcompress
\compresstrue

\newcommand{\varvspace}[1]{
\ifcompress
    \vspace{#1}
\fi
}
\begin{document}

\runningauthor{Waïss Azizian, Damien Scieur, Ioannis Mitliagkas, Simon Lacoste-Julien, Gauthier Gidel}

\twocolumn[

\aistatstitle{Accelerating Smooth Games by Manipulating Spectral Shapes}

\aistatsauthor{ Waïss Azizian$^{1,\dag}$\, Damien Scieur$^{2}$ \, Ioannis Mitliagkas$^{3,\ddagger}$ \, Simon Lacoste-Julien$^{2,3,\ddagger}$\, Gauthier Gidel$^{3}$}
\aistatsaddress{ $^1$École Normale Supérieure, Paris  $\;\; ^2$Mila \& SAIT AI Lab, Montreal   $\;\; ^3$Mila \& DIRO, Université de Montréal} ]

\begin{abstract}
We use matrix iteration theory to characterize acceleration in smooth games. We define the {\em spectral shape} of a family of games as the set containing all eigenvalues of the Jacobians of standard gradient dynamics in the family. Shapes restricted to the real line represent well-understood classes of problems, like minimization. Shapes spanning the complex plane capture the added numerical challenges in solving smooth games. In this framework, we describe gradient-based methods, such as extragradient, as transformations on the spectral shape. Using this perspective, we propose an optimal algorithm for bilinear games. For smooth and strongly monotone operators, we identify a continuum between convex minimization, where acceleration is possible using Polyak's momentum, and the worst case where gradient descent is optimal. Finally, going beyond first-order methods, we propose an accelerated version of consensus optimization. 
\end{abstract}

\section{Introduction}
Recent successes of multi-agent formulations in various areas of deep learning~\citep{goodfellowGenerativeAdversarialNets2014,pfauConnectingGenerativeAdversarial2016} have caused a surge of interest in the theoretical understanding of first-order methods for the solution of differentiable multi-player games~\citep{palaniappan2016stochastic,gidelVariationalInequalityPerspective2018a,balduzziMechanicsNPlayerDifferentiable2018b,meschederNumericsGANs2017a,meschederWhichTrainingMethods,mazumdarFindingLocalNash2019a}.
This exploration hinges on a key question:
\begin{center}
    \varvspace{-2.0mm}
    \emph{How fast can a first-order method be?}
    \varvspace{-2.0mm}
\end{center}
In convex minimization, \citet{nesterovMethodSolvingConvex1983,nesterovIntroductoryLecturesConvex2004} answered this question with lower bounds for the rate of convergence and an accelerated, momentum-based algorithm matching that optimal lower bound.

\input{figure_geometric.tex}

The dynamics of numerical methods is often described by a vector field, $F$, and summarized in the spectrum of its Jacobian. 
In minimization problems, the eigenvalues of the Jacobian lie on the real line.
On strongly convex problems, the {\em condition number} (the dynamic range of eigenvalues) is at the heart of Nesterov's upper and lower bound results, characterizing the hardness of an minimization problem.

Our understanding of differentiable games is nowhere close to this point.
There, the eigenvalues of the Jacobian at the solution are distributed on the complex plane,
suggesting a richer, more complex set of dynamics \citep{meschederNumericsGANs2017a,balduzziMechanicsNPlayerDifferentiable2018b}.
Some old papers \citep{g.m.korpelevichExtragradientMethodFinding1976,tsengLinearConvergenceIterative1995} and many recent ones \citep{nemirovskiProxMethodRateConvergence2004,chenAcceleratedSchemesClass2014,palaniappan2016stochastic,meschederNumericsGANs2017a,gidelVariationalInequalityPerspective2018a,gidelNegativeMomentumImproved2018b,daskalakisTrainingGANsOptimism2017a,mokhtariUnifiedAnalysisExtragradient2019,azizian2019tight} suggest new methods and provide better upper bounds.

All of the above work relies on bounding the magnitude or the real part of the eigenvalues of submatrices of the Jacobian.
This coarse-grain approach can be oblivious to the dependence of upper and lower bounds on the exact distribution of eigenvalues on the complex plane. 
More importantly, the questions of acceleration and optimality have not been answered for smooth games.

In this paper, we take a different approach. 
We use matrix iteration theory to characterize acceleration in smooth games.
Our analysis framework revolves around the {\em spectral shape} of a family of games, defined as the set containing all eigenvalues of the Jacobians of natural gradient dynamics in the family (cf.~\S\ref{sub:problem class}).
This fine-grained analysis framework can captures the dependence of upper and lower bounds on the specific shape of the spectrum.
Critically, it allows us to establish acceleration in specific families of smooth games.

\textbf{Contributions.}
Our main contribution is a geometric interpretation of the conditioning of a game (via its \emph{spectral shape} as illustrated in Fig.~\ref{figure:geometric_example}, and discussed with more details in \cref{sec:geometric_interpretation}). 
Our result links the ``hardness'' of a game to the distribution of the eigenvalues of its Jacobian of the game at the optimum. %
Using our framework, we make the following contributions.\\
    \textbf{1.} 
    We show a reduction from bilinear games to games with \textit{real} eigenvalues, where acceleration is possible through momentum. We provide lower bounds and design an optimal algorithm for this class. \\
    \textbf{2.} Showing that acceleration persists even if there is an ``imaginary perturbation", we propose an accelerated version of extragradient (EG) for bilinear games.\\
    \textbf{3.} 
    We accelerate \textit{consensus optimization} (CO), a \textit{cheap} second-order method. %
    We combine it with momentum to achieve a nearly-accelerated rate, improving the best rate previously known for this method.

\textbf{Organisation.} 
We recall the definition of the \textit{asymptotic convergence factor} in \cref{sec:lowerbounds} and use it to show that acceleration is \textit{not possible} for the general class of smooth and strongly monotone games. In \cref{sec:acceleration_in_game} we show that bilinear games or games with a ``small imaginary perturbation" can be accelerated. Finally, in \cref{sec:non_conventional} we improve the rate of CO by using momentum.

\section{Related work}
\vspace{-2mm}
\textbf{Matrix iteration theory.} There is extensive literature on iterative methods for linear systems, due to their countless applications. An important line of work considers the design of iterative methods through the lens of approximation problems by polynomials on the complex plane. \citet{eiermannConstructionSemiterativeMethods1983} then used complex analysis tools to define, for a given compact set, its \emph{asymptotic convergence factor}: it is the optimal asymptotic convergence rate a first-order method can achieve for all linear systems with spectrum in the set. %
\citet{niethammerAnalysisOfkstepIterative1983} bring tools from summability theory to analyze multi-step iterative methods in this framework and provide optimal methods, in particular, the momentum method for ellipses. \citet{eiermannStudySemiiterativeMethods1985} continued in this direction, summarizing and improving the previous results.
Finally \citet{eiermannHybridSemiIterativeMethods1989} study how polynomial transformations of the spectrum help compute the asymptotic convergence factor and the optimal method for a given set, potentially yielding faster convergence.

\textbf{Acceleration and lower bounds.}
Lower bounds of convergence are standard in convex optimization~\citep{nesterovIntroductoryLecturesConvex2004} but are often non-asymptotic or cast in an infinite-dimensional space. \citet{arjevaniLowerUpperBounds2016,arjevaniIterationComplexityOblivious2016a} showed non-asymptotic lower bounds using a framework called $p$-SCLI close to matrix iteration theory. \citet{ibrahimLowerBoundsConditioning2019,azizian2019tight} extended this framework to multi-player games, but they consider lower and upper-bounds on the eigenvalues of the Jacobian of the game rather than their distribution in the complex plane. Two main acceleration methods in convex optimization achieve these lower bounds, Polyak's momentum \citep{polyakMethodsSpeedingConvergence1964} and Nesterov's acceleration \citet{nesterovMethodSolvingConvex1983}. The latter is the only one that has global convergence guarantees for convex functions.
Nevertheless, Polyak's momentum still plays a crucial role in the training of large scale machine learning models \citet{sutskeverImportanceInitializationMomentum2013}. 

\textbf{Acceleration for games.}
Recent work applied acceleration techniques to game optimization. \citet{gidelNegativeMomentumImproved2018b} showed that negative momentum with alternating updates converges on bilinear games, but with the same geometrical rate as EG.  \citet{chenAcceleratedSchemesClass2014} provided a  version of the mirror-prox method which improves the constant but not its rate. In the context of minimax optimization, \citet{palaniappan2016stochastic} used Catalyst~\citep{lin2015universal}, a generic acceleration method, to improve the convergence of variance-reduced algorithms for min-max problems. In the context of variational inequalities, the standard assumptions on the operator are Lipschitzness and (strong) monotonicity~\citep{tsengLinearConvergenceIterative1995,nesterovDualExtrapolationIts2003}.
\citet{nemirovskiProxMethodRateConvergence2004} provided a lower bound in $\bigO(1/t)$ on the convergence rate for smooth monotone games, which suggests that EG is nearly optimal in the strongly monotone case.  In our work, we show that acceleration is possible by substituting the smoothness and monotonicity assumptions on the operator into more precise assumptions on the \emph{eigenvalues of its Jacobian}. 

\section{Setting and notation}\label{section: games}
\vspace{-2mm}
We consider the problem of finding a stationary point $\omega^* \in \R^d$ of a vector field $F:\R^d \rightarrow \R^d$, i.e., $F(\omega^*) = 0.$, the solution of an unconstrained \emph{variational inequality} problem~\citep{harker1990finite}. A relevant special case is a $n$-player convex game, where $\omega^*$ corresponds to a Nash equilibrium \citep{vonneumannTheoryGamesEconomic1944,balduzziMechanicsNPlayerDifferentiable2018b}. 
Consider $n$ players $i = 1,\ldots, n$ who want to minimize their loss $l_i(\omega^{(i)},\omega^{(-i)})$. The notation $\cdot^{(-i)}$ means all indexes but $i$. A Nash equilibrium satisfies
\[
    (\omega^*)^{(i)} \in \argmin_{\omega^{(i)}\in\mathbb{R}^{d_i}} l_i(\omega^{(i)},(\omega^*)^{(-i)}) \quad \forall i \in \{1,\ldots,n\}.
    \varvspace{-2mm}
\] 
In this situation no player can unilaterally reduce its loss. The vector field of the game is
\[
    F(\omega) = 
    \begin{bmatrix}
    \nabla_{\omega_1} l_1^T(\omega^{(1)},\omega^{(-1)}),...,
    \nabla_{\omega_n} l_n^T(\omega^{(n)},\omega^{(-n)})
    \end{bmatrix}^T.
    \varvspace{-2mm}
\]

\subsection{First-order methods}
To study lower bounds of convergence, we need a class of algorithms. We consider the classic definition\footnote{Technically, first-order algorithms are more generally methods that have access only to first-order oracles.} of first-order methods from \cite{nemirovskyProblemComplexityMethod1983}.
\begin{definition}\label{def: first-order method}
    A \emph{first-order method} generates
    \varvspace{-1mm}
    \[
        \omega_t \in \omega_0 + \Span\{ F(\omega_0),\, \ldots,\, F(\omega_{t-1}) \}\,, \quad  t\geq 1\,.
        \varvspace{-2mm}
    \]
\end{definition}
This class is widely used in large-scale optimization, as it involves only gradient computation. For instance, Nesterov's acceleration belongs to the class of first-order methods. On the contrary, this definition does not cover Adagrad \citep{duchi2011adaptive}, that could conceptually be also considered as first-order. This is due to the diagonal re-scaling, so $\omega_{t}$ can go \textit{outside} the span of gradients. The next proposition
gives a way to easily identify first-order methods that fit our definition.

\begin{prop}[{restate=[name=]propObliviousFirstOrder}]\label{prop:oblivious_first_order}\citep{arjevaniIterationComplexityOblivious2016a}
    first-order methods can be written as
    \begin{equation} \label{eq: first-order method}
        \textstyle \omega_{t+1} =\sum_{k = 0}^t \alpha_k^{(t)} F(\omega_k) + \beta_k^{(t)} \omega_k,
    \end{equation}
    where $\sum_{k=0}^t \beta_k^{(t)} = 1$. The method is called \textit{oblivious} if the coefficients $\alpha_k^{(t)}$ and $\beta_k^{(t)}$ are known \textit{in advance}.
\end{prop}
Oblivious methods allow the knowledge of ``side information" on the function, like its smoothness constant. Most of first-order methods belong to this class, but it excludes for instance methods with adaptive step-sizes. We show how standard methods fit into this framework.

\textbf{Gradient method.} Consider the gradient method with time-dependant step-size: $\omega_{t+1} = \omega_t - \eta_t F(\omega_t)$.
This is a first-order method, where $\alpha_t^{(t)} = -\eta_t$, $\beta_t^{(t)} = 1$ and all the other coefficients set to zero.

\textbf{Momentum method.} The momentum method defines iterates as $\omega_{t+1} = \omega_t - \alpha F(\omega_t) + \beta(\omega_t - \omega_{t-1})$. It fits into the previous framework with $\alpha_t^{(t)} = -\alpha$, $\beta_t^{(t)} = 1 + \beta$, $\beta_{t-1}^{(t)} = - \beta$.%

\textbf{Extragradient method.} Though slightly trickier, the extragradient method (EG) is also encompassed by this definition. The iterates of EG are defined by $\omega_{t+1} = \omega_t - \eta F(\omega_t - \eta F(\omega_t))$ where
\begin{align*}
    & \begin{cases}
        \beta_t^{(t)} = 0,\,\,\beta_{t-1}^{(t)} = 1 \quad \text{if $t$ is odd (update)}\,,\\
        \beta_t^{(t)} = 1,\,\, \beta_{t-1}^{(t)} = 0 \quad \text{if $t$ is even (extrapolation)} \,,
    \end{cases}
    \varvspace{-1mm}
\end{align*}
and $\alpha_t^{(t)} = -\eta$ the step size.%

\medskip
The next (known) lemma shows that when $F$ is linear, first-order methods can be written using \textit{polynomials}.
\begin{lemma}[{restate=[name=]lemmaFirstOrderMethod}]\label{lemma: first-order method}\citep[e.g.][]{chihara2011introduction}
    If $F(\omega) = A\omega + b$,
    \varvspace{-1mm}
    \begin{equation}\label{eq:iterates_polynomial}
        \omega_{t} - \omega^* = p_t(A)(\omega_0 - \omega^*) \,,
        \varvspace{-1mm}
    \end{equation}
    where $\omega^*$ satisfies $A\omega^*+b = 0$ and $p_t$ is a real polynomial of degree at most $t$ such that $p_t(0) = 1$.
\end{lemma}
We denote by $\mathcal{P}_t$ the set of real polynomials of degree at most $t$ such that $p_t(0) = 1$. Hence, the convergence of a first-order method can be analyzed through the sequence of polynomials $(p_t)_t$ it defines.

\subsection{Problem class}
\label{sub:problem class}
In the previous section, when $F$ is the linear function $F = Ax+b$, the iterates $\omega_t$ follow the relation \eqref{eq:iterates_polynomial} involving the polynomial $p_t$. Since all first-order methods can be written using polynomials~\eqref{eq: first-order method}, they follow
\begin{equation} \label{eq:convergence_rate}
    \|\omega_t-\omega^*\|_2 = \|p_t(A)(\omega_0-\omega^*)\|_2 \,.
\end{equation}
This gives the rate of convergence of the method for a specific matrix $A$. Instead, we consider a larger class of problems. It consists of a set $\mathcal{M}_K$ of matrices $A$ whose eigenvalues belong to a set $K$ on the complex plane,
\begin{equation} \label{eq:class_problem}
    \mathcal{M}_K \coloneqq \{ A \in \mathbb{R}^d :\; \Sp(A) \subset K \subset \mathbb{C}_+\},
\end{equation}
where $\Sp(A)$ is the set of eigenvalues of $A$ and $\mathbb{C}_+$ is the set of complex numbers with positive real part. Moreover, we assume that $d \geq 2$ to avoid trivial cases.

\subsection{Geometric intuition} \label{sec:geometric_interpretation}

Our paper is entirely based on the study of the support $K$ of the eigenvalues of the Jacobian of the operator $F$, denoted by $\J_F(\omega^*)$. Before detailing our theoretical results, we give a high-level explanation of our objectives. This geometric intuition comes from the fact that the standard assumptions made in the literature correspond to particular problem classes $\mathcal M_K$.

\textbf{Smooth and strongly convex minimization.}
Consider the minimization of a twice-differentiable, $L$-smooth and \mbox{$\mu$-strongly} convex function $f$, 
\[
    \mu \textbf{I} \preceq \nabla^2 f(\omega) \preceq L\textbf{I} \quad \forall \omega\in \mathbb{R}^d.
\]
There is a link between minimization problems and games, since the vector field $F$ becomes the gradient of the objective, and its Jacobian $\J_F(\omega)$ is the Hessian $\nabla^2 f(\omega)$. Thus, the class corresponding to the minimization of smooth, strongly convex functions is
\begin{equation*}
    \{F:\, \forall \omega\in\mathbb{R}^d,\,\,\Sp \J_F(\omega) \subset [\mu, L]\}\,,\;\;  0 < \mu \leq L\}.
    \varvspace{-2mm}
\end{equation*}

\textbf{Bilinear games.} Consider the following problem, 
\varvspace{-1mm}
\begin{equation*}
    \min_{x \in \R^d}\max_{y \in \R^d} x^\top A y \,.
    \varvspace{-1mm}
\end{equation*}
Its Jacobian $\J_F(\omega)$ is constant and skew-symmetric. It is a standard linear algebra result (see \cref{lemma: spectrum bilinear game}) to show that $\Sp \J_F(\omega) \in \pm [i\sigma_{\min}(A),i\sigma_{\max}(A)]$.

\textbf{Variational inequalities.}
The Lipchitz assumption
\begin{equation} \label{eq:Lipschitz}
    \|F(\omega) - F(\omega')\|_2^2 \leq L \|\omega - \omega'\|_2^2
\end{equation}
implies an upper bound on the magnitude of the eigenvalues of $\J_F(\omega^*)$. The strong monotonicity assumption 
\begin{equation} \label{eq:StrongMonotone}
(\omega-\omega')^T(F(\omega)-F(\omega')) \geq \mu\|\omega-\omega'\|_2^2
\end{equation}
implies a lower bound on the real part of the eigenvalues of $\J_F(\omega^*)$ (see \cref{lemma: strongly monotone problem class} in \S\ref{section: proof of general lemmas}) which thus belong to
\begin{equation*}
     K = \{ \lambda \in \mathbb{C}:\, 0<\mu\leq\Re \lambda, \, |\lambda| \leq L  \}.
\end{equation*}
This set is the intersection between a circle and a half-plane, as shown in Figure~\ref{figure: visual proof lb} (left).

\textbf{Fine-grained bounds.} \citet{nemirovskiProxMethodRateConvergence2004} provides a lower-bound for the class of strongly monotone and Lipschitz operators (see \S\ref{subsection: asymptotic convergence factor}) excluding the possibility of acceleration in that general setting.  It motivates the adoption of more refined assumptions on the eigenvalues of $\J_F(\omega^*)$. We consider the class of games where these eigenvalues belong to a specified set $K$. Since $\J_F(\omega^*)$ is real, its spectrum is symmetric w.r.t.~the real axis, so we assume that $K$ is too. For this class of problem, we have a simple method to compute lower and upper convergence bounds using a class of well studied shapes: ellipses. 

\begin{prop}[Ellipse method for lower and upper bound (Informal)]
\label{prop:ellipses method}
Let $K \subset \mathbb{C}_+$ be a compact set, then any ellipse symmetric w.r.t.~the real axis that includes (resp. is included in) $K$ provides an upper (resp. lower) convergence bound for the class of problem $\mathcal M_K$ using Polyak momentum with a step-size and a momentum depending on the ellipse.  
\end{prop}
See \cref{subsection: app ellipses}, \cref{thm: ellipses} for the precise result on ellipses. The proposition extends to any shape whose optimal algorithm (resp. lower bound) is known. This proposition, illustrated in Fig.~\ref{figure:geometric_example}, heavily relies on the fact that, the optimal method for ellipses is Polyak momentum \citep{niethammerAnalysisOfkstepIterative1983}.

Any first-order method can be seen as a way to transform the set $K$. In order to illustrate that we consider Lemma~\ref{lemma: first-order method}: since a first-order method update for a linear operator $F = Ax + b$ can be written using a polynomial $p$, the eigenvalues to consider are not the ones of $A$ but the ones of $p(A)$. Thus, the set of interest is $p(K)$.

As an example, consider EG with momentum. This consists in applying the momentum method to the transformed vector field $\omega \mapsto F(\omega - \eta F(\omega))$. From a spectral point of view, this is equivalent to first transforming the shape $K$ into $\varphi(K)$ with the extragradient mapping $\varphi_\eta:\lambda \mapsto \lambda(1-\eta\lambda)$, then study the effect of momentum on $\varphi(K)$. This example of transformation is illustrated in Fig.~\ref{figure:geometric_example}, and this idea is used in \cref{subsection: extrgradient bilinear}.

\section{Asymptotic convergence factor} \label{sec:lowerbounds}

We recall known results that compute lower bounds for some classes of games using the \textit{asymptotic convergence factor} \citep{eiermannConstructionSemiterativeMethods1983,eiermannStudySemiiterativeMethods1985, nevanlinnaConvergenceIterationsLinear1993}. Then, we illustrate them on two particular classes of problems.

\subsection{Lower bounds for a class of problems}

We now show how to lower bound the worst-case rate of convergence of a \textit{specific} method over the class $\mathcal{M}_K$ \eqref{eq:class_problem}, with the worst possible initialisation $\omega_0$. We start with equation \eqref{eq:convergence_rate}, but this time we pick the worst-case over all matrices $A\in \mathcal{M}_K$, i.e.,
\varvspace{-1mm}
\[
    \max_{A\in \mathcal{M}_K} \|p_t(A)(\omega_0-\omega^*)\|_2.
    \varvspace{-1mm}
\]
Now, we can pick an arbitrary bad initialisation $\omega_0$, in particular, the one that corresponds to the largest eigenvalue of $p_t(A)$ in magnitude. This gives \varvspace{-1mm}
\begin{align}
    \exists \omega_0 : \|\omega_t-\omega^*\|_2 & \geq \max_{A\in \mathcal{M}_K} \rho\Big(p_t(A)\Big)\|\omega_0-\omega^*\|_2 \nonumber\\
    &  = \max_{\lambda \in K} |p_t(\lambda)| \|\omega_0-\omega^*\|_2 \,.
    \label{eq:lower_bound_specific_method}
    \varvspace{-1mm}
\end{align}
It remains to lower bound $\max_{\lambda \in K} |p_t(\lambda)|$ over \emph{all possible} first-order methods. This is called the \textit{asymptotic convergence factor}, presented in the next section.

\subsection{Asymptotic convergence factor}\label{subsection: asymptotic convergence factor}
Here we recall the definition of the \textit{asymptotic convergence factor} \citep{eiermannConstructionSemiterativeMethods1983}, which gives a lower bound for the rate of convergence over matrices which belong to the class $\mathcal{M}_k$ \eqref{eq:class_problem}, for all possible first-order methods. We mainly follow the definition of \citet{nevanlinnaConvergenceIterationsLinear1993} (see Rmk.~\ref{rmk: definition acf litterature} in \S\ref{section: proof of general lemmas} for details).

The simplest way to lower bound $\|\omega_t-\omega^*\|_2$ is given by minimizing \eqref{eq:lower_bound_specific_method} over all polynomials corresponding to a first-order method. By Lemma \ref{lemma: first-order method}, this class of polynomials is given by $\mathcal{P}_t$. Thus, for some $\omega_0$,
\varvspace{-1mm}
\[
    \|\omega_t-\omega^*\| \geq \min_{p_t\in \mathcal{P}_t}  \max_{\lambda \in K} |p_t(\lambda)| \cdot \|\omega_0-\omega^*\|_2 .
    \varvspace{-1mm}
\]
The \textit{asymptotic convergence factor} $\acf(K)$ for the class $K$ is given by taking the \textit{minimum average} rate of convergence over $t$ for any $t$, i.e.,
\varvspace{-1mm}
\begin{equation}\label{eq: asymptotic convergence factor}
    \acf(K) = \inf_{t>0} \min_{p_t\in \mathcal{P}_t}  \max_{\lambda \in K}  \sqrt[t]{|p_t(\lambda)|} \,.
    \varvspace{-1mm}
\end{equation}
This way, by construction, $\acf(K)$ gives a lower-bound on the \textit{worst-case} rate of convergence for the class $\mathcal{M}_K$. We formalize this statement in the proposition below.

\begin{prop}[{restate=[name=]propKappaLowerBound}]\label{prop: kappa lower bound}\citep{nevanlinnaConvergenceIterationsLinear1993}
Let $K \subset \C$ be a subset of $\C$ symmetric w.r.t.~the real axis, which does not contain $0$ and such that $\mathcal{M}_K \neq \emptyset$. Then, any oblivious first-order method (whose coefficients only depend on $K$) satisfies,
\varvspace{-1mm}
\[
    \forall t \geq 0,\,\exists A\in\mathcal{M}_K,\, \exists \omega_0 : \|\omega_t-\omega^*\|_2 \geq \acf(K)^t\|\omega_0-\omega^*\|_2.
    \varvspace{-2mm}
\]

\end{prop}

However, the object $\acf(K)$ may be complicated to obtain as it depends on the solution of a minimax problem \textit{over a set $K\subset{\mathbb{C}_+}$}. If the set is simple enough, we can lower-bound the asymptotic rate of convergence. We start by giving the two extreme cases: when $K$ is a segment on the real line (convex and smooth minimization) or $K$ is a disc (monotone and smooth games). 

\subsection{Extreme cases: real segments and discs}
\label{subsection: lower bound strongly monotone games}
\textbf{Smooth and strongly convex minimization.}

In the case where we are interested in lower-bounds, we can consider the restricted class of functions where $J_F(\omega)(=\nabla^2 f(\omega))$ is constant, i.e., independent of $\omega$. This corresponds to quadratic minimization, and our restricted class becomes
\varvspace{-2mm}
\[
    \mathcal{M}_K \quad \text{where}\;\; K=[\mu,L].
    \varvspace{-1mm}
\]
For this specific class, where $K$ is a segment in the real line, the solution to the subproblem associated to the \textit{asymptotic rate of convergence} \eqref{eq: asymptotic convergence factor}, i.e.,
\varvspace{-1mm}
\begin{equation}\label{eq: optimization problem polynomial real axis}
    \min_{p \in \P_t}\max_{\lambda \in [\mu, L]} |p(\lambda)|
    \varvspace{-1mm}
\end{equation}
is well-known. The optimal polynomial $p_t^*$ is a properly scaled and translated Chebyshev polynomial of the first kind of degree $t$ \citep{golubChebyshevSemiiterativeMethods1961,manteuffelTchebychevIterationNonsymmetric1977}. The rate of convergence of $p_t$ evolves with $t$, but asymptotically converges to
\varvspace{-1mm}
\begin{equation*}
    \textstyle \acf([\mu, L]) = \frac{\sqrt L - \sqrt \mu}{\sqrt L + \sqrt \mu} \,.
    \varvspace{-1mm}
\end{equation*}
This is the lower bound of \citet[Thm.~2.1.13]{nesterovIntroductoryLecturesConvex2004}, which corresponds to an accelerated linear rate. The condition number $L/\mu$ appears as a square root unlike for the rate of the plain gradient descent, which implies a huge (asymptotic) improvement. 

In this section, we have seen that when the spectrum is constrained to be on a segment in the real line, one can expect acceleration. The next section shows that this is not the case for the class of discs.

\textbf{Discs and strongly monotone vector fields}
\label{subsection: lower bound strongly montone games}
Consider a disc with a real positive center
\begin{equation*}
K = \{z \in \C:|z - c| \leq r\}, \quad\text{with } 0 < c < r.
\end{equation*}
This time again, the shape is simple enough to have an explicit solution for the optimal polynomials 
\varvspace{-1mm}
\[
    p_t^*(\lambda) = \argmin _{p_t \in \P_t} \max_{\lambda \in K} |p_t(\lambda)|.
    \varvspace{-2mm}
\]
In this case, the optimal polynomial reads $p_t^*(\omega) = (1 - \omega/c)^t$, and this corresponds to gradient descent with step-size $\eta = 1/c$. Hence, with this specific shape, gradient method is optimal \citep[\S 6.2]{eiermannStudySemiiterativeMethods1985}; \citet[Example 3.8.2]{nevanlinnaConvergenceIterationsLinear1993}. %
A direct consequence of this result is a lower bound of convergence for the class of Lipshitz, strongly monotone vector fields, i.e., vector fields $F$ that satisfies~\eqref{eq:Lipschitz}-\eqref{eq:StrongMonotone}.
For linear vector fields parameterized by the matrix $A$ as in Lemma \ref{lemma: first-order method}, this is included in the set
\begin{equation}\label{eq: strongly montone problem class}
    \mathcal{M}_K,\,\, K = \{ \lambda \in \mathbb{C}:\, 0<\mu\leq\Re \lambda, \, |\lambda| \leq L  \}.
\end{equation}
This set is the intersection between a circle and a half-plane, as shown in Figure~\ref{figure: visual proof lb} (left). Notice that the disc of center $\frac{\mu + L}{2}$ and radius $\frac{L - \mu}{2}$ actually fits in $K$, as illustrated by Fig.~\ref{figure: visual proof lb}. Since this disc in \textit{included} in $K$, a lower bound for the disc also gives a lower bound for $K$, as stated in the following corollary.

\begin{figure}
    \centering
    \setcounter{t}{80}
    \setcounter{L}{10}
    \hspace{-4mm}
    \begin{tikzpicture}[scale=0.65]
        \begin{axis}[
            grid=major,
            axis equal image,
            yticklabel={
            $\pgfmathprintnumber{\tick}i$
            },
        	xmin=0,   xmax={\value{L} + 2},
        	ymin={-\value{L} - 1},   ymax={\value{L} + 1},
        	]
           \draw [domain=-\value{t}:\value{t}, draw=yellow, thick, fill=yellow!50!white, fill opacity=0.5, samples=65] plot (axis cs: {\value{L}*cos(\x)}, {\value{L}*sin(\x)});
           \draw [draw=yellow, thick](axis cs: {\value{L}*cos(\value{t})}, {\value{L}*sin(\value{t})}) -- (axis cs: {\value{L}*cos(\value{t})}, {\value{L}*sin(-\value{t})});
           \draw [domain=-180:180, draw=red, fill=red!50!white, thick, fill opacity=0.5, samples=65] plot (axis cs: {\value{L}*(1 + cos(\value{t}))/2 + cos(\x) * \value{L}*(1 - cos(\value{t}))/2}, {sin(\x) * \value{L}*(1 - cos(\value{t}))/2});
           \node[anchor=south east] (mu) at (axis cs: {\value{L} * cos(\value{t})}, 0) {$\mu$};
           \draw [fill=black] (axis cs: {\value{L} * cos(\value{t})}, 0) circle (1pt);
           \node [circle, anchor=south west] (L) at (axis cs: {\value{L}}, 0) {$L$};
           \draw [fill=black] (axis cs: {\value{L}}, 0) circle (1pt);
       \end{axis}
    \end{tikzpicture}
    \hspace{-3mm}
    \newcounter{mu}
    \setcounter{mu}{2}
    \setcounter{L}{28}
    \begin{tikzpicture}[scale=0.65]
        \begin{axis}[
            grid=major,
            axis equal image,
            yticklabel={
            $\pgfmathprintnumber{\tick}i$
            },
        	xmin=0,   xmax={\value{L} + 2},
        	ymin={-\value{L}/2},   ymax={\value{L}/2},
        	]
           \draw [domain=-180:180, thick, draw=blue, fill=blue!50!white, fill opacity=0.5, samples=65] plot (axis cs: {(\value{L} + \value{mu})/2 + cos(\x) * (\value{L} - \value{mu})/2}, {sin(\x) * sqrt(\value{mu}*\value{L})});
           \draw [domain=0:1, thick, draw=blue!80!white] plot (axis cs:{(\x * \value{mu} + (1 - \x) * \value{L}}, 0);
           \draw [domain=-1:1, thick, draw=blue!80!white] plot (axis cs:{0.5 * (\value{mu} + \value{L})},{\x * sqrt(\value{mu}*\value{L})});
            \node[anchor=south east] (mu) at (axis cs: {\value{mu}}, 0) {$\mu$};
           \draw [fill=black] (axis cs: {\value{mu}}, 0) circle (1pt);
           \node [anchor=south west] (L) at (axis cs: {\value{L}}, 0) {$L$};
           \draw [fill=black] (axis cs: {\value{L}}, 0) circle (1pt);
           \node [anchor=west] (epsilon) at (axis cs:{0.5 * (\value{mu} + \value{L})},{0.5 * sqrt(\value{mu}*\value{L})}) {\textcolor{blue}{$\epsilon$}};
        \end{axis}
    \end{tikzpicture}
    \hspace{-4mm}
    \caption{\small
    \textbf{Left:} Illustration of the proof of \cref{cor: lb strongly monotone}. The yellow set correspond to $K$, the set of strongly monotone problems while the red disc is the disc of center $\frac{1}{2}(\mu + L)$ and radius $\frac{1}{2}(L- \mu)$ which fits inside.
    \textbf{Right:} Illustration of $K_\epsilon$ of \cref{prop: perturbed acceleration} with $\epsilon = \sqrt{\mu L}.$
    \varvspace{-2ex}
    }\label{figure: visual proof lb}
\end{figure}

\begin{cor}\label{cor: lb strongly monotone}
Let $K$ be defined in \eqref{eq: strongly montone problem class}. Then,
\varvspace{-1mm}
\begin{equation*}
    \textstyle \acf(K) > \frac{L - \mu}{L + \mu} = 1 - \frac{2\mu}{L + \mu}\,.
    \varvspace{-2mm}
\end{equation*}
\end{cor}

The rate of \cref{cor: lb strongly monotone} is already achieved by first-order methods, without momentum or acceleration, such as EG. Thus, acceleration is \textit{not possible} for the general class of smooth, strongly monotone games.

\section{Acceleration in games}
\label{sec:acceleration_in_game}

We present our contributions in this section. The previous section highlights a big contrast between optimization and games. In the former, acceleration is possible, but this does not generalize for the latter. Here, we explore acceleration via a sharp analysis of intermediate cases, like imaginary segments (bilinear games) or thin ellipses (perturbed acceleration), via lower and upper bounds. Since we use spectral arguments, the convergence guarantees of our algorithms are local, but lower bounds remain valid globally.

\subsection{Local convergence of optimization methods for nonlinear vector fields}

Before presenting our result, we recall the classical local convergence theorem from \citet{polyakMethodsSpeedingConvergence1964}. In this section, we are interested in finding the fixed point $\omega^*$ of a vector field $V$, i.e, $V(\omega^*) = \omega^*$. $V$ here plays the role of an iterative optimization methods and defines iterates according to the fixed-point iteration
\varvspace{-1mm}
\begin{equation}
    \omega_{t+1} = V(\omega_t). \label{eq:power method}
    \varvspace{-2mm}
\end{equation}

Analysing the properties of the vector field $V$ is usually challenging, as $V$ can be any nonlinear function. However, under mild assumption, we can simplify the analysis by considering the linearization
    $\textstyle V(\omega) \approx V(\omega^*) + \mathbf{J}_V(\omega^*)(\omega - \omega^*)$,
where $\mathbf{J}_V(\omega)$ is the Jacobian of $V$ evaluated at $\omega^*$. 
The next theorem shows we can deduce the rate of convergence of \eqref{eq:power method} using the spectral radius of $\mathbf{J}_V(\omega^*)$, denoted by $\rho(\mathbf{J}_V(\omega^*))$.
\begin{thm}[\citet{polyakIntroductionOptimization1987a}]\label{thm: local convergence}
Let $V: \R^d \longrightarrow \R^d$ be continuously differentiable and let $\omega^*$ one of its fixed-points. 
Assume that there exists $\rho^* > 0$ such that,
\begin{equation*}
    \rho(\mathbf{J}_V(\omega^*)) \leq \rho^* < 1.
\end{equation*}
For $\omega_0$ close to $\omega^*$, \eqref{eq:power method} converges linearly to $\omega^*$ at a rate $\bigO((\rho^*+\epsilon)^t)$. If $V$ is linear, then $\epsilon = 0$.
\end{thm}

Recent works such as \citet{meschederNumericsGANs2017a,gidelNegativeMomentumImproved2018b, daskalakisLimitPointsOptimistic2018} used this connection to study game optimization methods.

\cref{thm: local convergence} can be applied directly on methods which use only the last iterate, such as gradient or EG. For methods that do not fall into this category, such as momentum, a small adjustment is required, called \textit{system augmentation}.

Consider that $V:\R^d \times \R^d \rightarrow \R^d$ follows the recursion
\begin{equation}
    \omega_{t+1} = V(\omega_t, \omega_{t-1}). \label{eq:power method momentum}
\end{equation}
Instead we consider its \emph{augmented operator}
\[
    \begin{bmatrix}
        \omega_{t}\\
        \omega_{t+1}
    \end{bmatrix}
    = V_{\text{augm}}(\omega_t,\omega_{t-1}) = 
    \begin{bmatrix}
        \omega_t\\
        V(\omega_t, \omega_{t-1})
    \end{bmatrix},
\]
to which we can now apply the previous theorem. This technique is summarized in the following lemma.
\begin{lemma}[{restate=[name=]lemmaAugmentedOperatorSpectralRate}]
\label{lemma: local convergence momentum}
Let $V: \R^d \times \R^d \longrightarrow \R^d$ be continuously differentiable and let $\omega^*$ satisfies
$V(\omega^*,\omega^*) = \omega^*\,.$
Assume there exists $\rho^* > 0$ such that,
$
    \rho(\mathbf{J}_{V_{\text{augm}}}(\omega^*)) \leq \rho^* < 1
$.
If $\omega_0$ and $\omega_1$ are close to $\omega^*$, then \eqref{eq:power method} converges linearly to $\omega^*$ at rate $\big(\rho^*+\epsilon\big)^t$. If $V$ is linear, then $\epsilon = 0$.
\end{lemma}

\subsection{Acceleration for bilinear games}
For convex minimization, adding momentum results in an accelerated rate for strongly convex functions we have discuss above. For instance, if $ \Sp \nabla F(\omega^*) \subset [\mu, L]$, the Polyak's Heavy-ball method (see the full statement in\cref{subsection: app bilinear games}), \citet[Thm.~9]{polyakMethodsSpeedingConvergence1964}
\begin{align}
    \omega_{t+1} & = V^{\text{Polyak}}(\omega_{t}, \omega_{t-1}) \nonumber \\
     & := \omega_t - \alpha F(\omega_t) + \beta(\omega_t - \omega_{t-1}) \label{eq:polyak_hb}
\end{align}
converges (locally) with the accelerated rate
\begin{equation*}
     \textstyle \rho(\J_{V^{\text{Polyak}}}(\omega^*, \omega^*)) \leq \frac{\sqrt L - \sqrt \mu}{\sqrt L + \sqrt \mu}\,.
\end{equation*}

Another example are bilinear games. Most known methods converge at a rate of $(1 - c\sigma_{min}(A)^2/\sigma_{\max}(A)^2)^t$ for some $c > 0$  \citep{daskalakisTrainingGANsOptimism2017a,meschederNumericsGANs2017a, gidelVariationalInequalityPerspective2018a, gidelNegativeMomentumImproved2018b, liangInteractionMattersNote2018,abernethyLastiterateConvergenceRates2019a}. Using results from \citet{eiermannHybridSemiIterativeMethods1989}, we show that this rate is suboptimal.

For bilinear games, the eigenvalues of the Jacobian $\J_F$ are  purely imaginary (see \cref{lemma: spectrum bilinear game} in\cref{subsection: app bilinear games}), i.e.,
\begin{equation*}
    K = [i\sigma_{\min}(A), i\sigma_{\max}(A)] \cup [-i\sigma_{\min}(A), -i\sigma_{\max}(A)].
\end{equation*}
A method that follows strictly the vector field $F$ does not converge, as its flow is composed by only concentric circles, thus leading to oscillations. This problem is avoided if we transform the vector field into another one with better properties. For example, the transformation 
\begin{equation}\label{eq:real_transform}
    F^{\text{real}}(\omega) = \tfrac{1}{\eta}(F(\omega  - \eta F(\omega)) - F(\omega))
\end{equation}
can be seen as a finite-difference approximation of $\nabla \left(\tfrac{1}{2}\|F\|_2^2\right)$. It is easier to find the equilibrium of $V$ since the eigenvalues of $\J_V(\omega) =  - \J_F^2(\omega)$ are located on a real segment. Thus, we can use standard minimization methods like the Polyak Heavy-Ball method.

\begin{prop}[{restate=[name=]propBilinearOpt}]\label{prop: bilinear opt}
Let $F$ be a vector field such that $\Sp \nabla F(\omega^*) \subset [ia, ib] \cup [-ia, -ib]$, for $0<a<b$. Setting $\sqrt{\alpha}= \frac{2}{a+b}, \sqrt{\beta} = \frac{b -a}{b+a}$, the Polyak Heavy-Ball method \eqref{eq:polyak_hb} on the transformation \eqref{eq:real_transform}, i.e.,
    \varvspace{-1mm}
\begin{align*}
    \omega_{t+1} 
    = \omega_t - \alpha F^{\text{real}}(\omega_t) + \beta(\omega_t - \omega_{t-1})\,.
    \varvspace{-1mm}
\end{align*}
converges locally at a linear rate $O\big((1 - \frac{2a}{a +b})^t\big)$.
\end{prop}
Using results from \citet{eiermannHybridSemiIterativeMethods1989}, we show that this method is optimal. Indeed, for this set, we can compute explicitly $\acf(K)$ from \eqref{eq: asymptotic convergence factor}, the lower bound for the local convergence factor.
\begin{prop}[{restate=[name=]propKappaBilinear}]\label{prop: kappa bilinear}
    Let $K = [ia, ib] \cup [-ia, -ib]$ for $0 < a < b$. Then,
    $\acf(K) = \sqrt{\frac{b - a }{b + a}}.$
\end{prop}

\begin{proof}
(Sketch). The transformation that we have applied, i.e.~$\lambda \mapsto -\lambda^2$, preserves the asymptotic convergence factor $\acf$ (up to a square root), as it satisfies the assumptions of \citet[Thm.~6]{eiermannHybridSemiIterativeMethods1989}.
\end{proof}
The difference of a square root between the lower bound and the bound on the spectral radius is explained by the fact that the method presented here queries two gradient per iteration and so one of its iterations actually corresponds to two steps of a first-order method as defined in Definition~\ref{def: first-order method}.

In this subsection, we showed that when the eigenvalues of the Jacobian are purely real or imaginary, acceleration is possible using momentum on the right vector field. Yet the previous subsection shows it is not the case for general smooth, strongly monotone games. The question of acceleration remains for intermediate shapes, like ellipses. The next subsection shows how to recover an accelerated rate of convergence in this case.

\subsection{Perturbed acceleration}\label{subsection: ellipses}

As we cannot compute $\acf$ explicitly for most sets $K$, we focus on ellipses to answer this question. They have been well studied, and optimal methods are again based on Chebyshev polynomials \citep{manteuffelTchebychevIterationNonsymmetric1977}. 

In this section we study games whose eigenvalues of the Jacobian belong to a thin ellipse. These ellipses correspond to the real segments $[\mu, L]$ perturbed in an elliptic way, see Fig.~\ref{figure: visual proof lb} (right). Mathematically, we have for $0 < \mu  < L$ and $\epsilon > 0$, the equation
\begin{equation*}
    \textstyle K_\epsilon = \bigg\{z \in \C: \bigg(\frac{\Re z - \frac{\mu + L}{2}}{\frac{L-\mu}{2}}\bigg)^2 + \left(\frac{\Im z}{\epsilon}\right)^2 \leq 1\bigg\} 
\end{equation*}
When  $\epsilon = 0$ (with the convention that $0/0 = 0$), Polyak momentum achieves the rate of \mbox{$1 - 2\frac{\sqrt \mu}{\sqrt \mu + \sqrt L}$}. However, when $\epsilon = \frac{L-\mu}{2}$, we showed the lower bound of $1 - 2\frac{\mu}{\mu + L}$ in \cref{cor: lb strongly monotone}. To check if acceleration still persists for intermediate cases, we study the behaviour of the asymptotic convergence factor (when $L/\mu \rightarrow +\infty$) as a function of $\epsilon$. The next proposition uses results from \citet{niethammerAnalysisOfkstepIterative1983,eiermannStudySemiiterativeMethods1985} to show that acceleration is still possible on $K_{\epsilon}$.

\begin{prop}[{restate=[name=]propPerturbedAcceleration}]\label{prop: perturbed acceleration}
Define $\epsilon(\mu, L)$ as $\frac{\epsilon(\mu, L)}{L} = \left(\frac{\mu}{L}\right)^{\theta}$ with $\theta > 0$ and $a \wedge b = \min(a,b)$. Then, when $\frac{\mu}{L} \rightarrow 0$, 
\begin{equation}
\acf(K_{\epsilon}) = \begin{cases}
 1 - 2\sqrt{\frac{\mu}{L}} + \bigO\left(\left(\frac{\mu}{L}\right)^{\theta \wedge 1}\right),& if \; \theta > \frac{1}{2}\\
 1 - 2(\sqrt 2 - 1)\sqrt{\frac{\mu}{L}} + \bigO\left(\frac{\mu}{L}\right),& if \; \theta = \frac{1}{2}\\
1 - \left(\frac{\mu}{L}\right)^{1-\theta} + \bigO\left(\left(\frac{\mu}{L}\right)^{1 \wedge (2 - 3\theta)}\right),& if \; \theta < \frac{1}{2}.
    \end{cases} \notag
\end{equation}
Moreover, the momentum method is optimal for $K_\epsilon$. This means
there exists $\alpha > 0$ and $\beta > 0$ (function of $\mu$, $L$ and $\epsilon$ only) such that if $\Sp \J_F(\omega^*) \subset K_\epsilon$, then,
$\rho(\J_{V^{\text{Polyak}}}(\omega^*, \omega^*)) \leq \acf(K_\epsilon)$.
\end{prop}

This shows that the convergence rate interpolates continuously between the accelerated rate and the non-accelerated one. Crucially, for small perturbations, that is to say if the ellipse is thin enough,
acceleration persists until $\theta = \frac{1}{2}$ or equivalently $\epsilon \sim \sqrt{\mu L}$.
That's why \cref{prop: perturbed acceleration} plays a central role in our forthcoming analyses of accelerated EG and CO.

\subsection{Accelerating extragradient}\label{subsection: extrgradient bilinear}

We now consider the acceleration of EG using momentum. Its main appealing property is its convergence on bilinear games, unlike the gradient method. On the class of bilinear problems, EG achieves a convergence rate of $(1 - ca^2/b^2)$ for some constant $c>0$.

In the previous section, we achieved an accelerated rate on bilinear games by applying momentum to the operator $F^{\text{real}}(\omega)$ instead of $F$, as the Jacobian of $F^{\text{real}}$ has real eigenvalues when $\J_F(\omega^*)$ has its spectrum in $K$. Here we try to apply momentum to the EG operator $F^{\text{e-g}}(\omega)$, defined as
    \varvspace{-1mm}
\begin{equation}\label{eq:extragradient_operator}
    F^{\text{e-g}}(\omega) =  F(\omega - \eta F(\omega)) \,.
    \varvspace{-1mm}
\end{equation}
Unfortunately, when $\Sp \J_F \subset K$, the spectrum of $F^{\text{e-g}}(\omega^*)$ is never purely real. Using the insight from  \cref{prop: perturbed acceleration}, we can choose $\eta > 0$ such that we are in the first case of \cref{prop: perturbed acceleration}, making acceleration possible.
\begin{prop}[{restate=[name=]propExtragradient}]\label{prop: extragradient}
Consider the vector field $F$, where $\Sp \J_F(\omega^*) \subset [ia,ib]\cup[-ia, -ib]$ for $0<a<b$. There exists $\alpha, \beta, \eta > 0$ such that, the operator defined by 
\begin{align*}
    \omega_{t+1} 
    & = \omega_t - \alpha F(\omega_t - \eta F(\omega_t)) + \beta (\omega_t - \omega_{t-1})\,,
\end{align*}
converges locally at a linear rate $O\big(\big(1 - c{\frac{a}{b}} + M\frac{a^2}{b^2}\big)^t\big)$ where $c = \sqrt{2}-1$ and $M$ is an absolute constant.  
\end{prop}
One drawback is that, to achieve fast convergence on bilinear games, one has to tune the two step-sizes $\alpha,\,\eta$ of EG precisely and separately. They actually differ by a factor $\frac{b^2}{a^2}$: $\eta$ is roughly proportional to $\frac{1}{a}$ while $\alpha$ behaves like $\frac{a}{b^2}$ (see \cref{lemma: asymptotic behavior of step sizes} in \cref{subsection: app extragradient}).

\section{Beyond typical first-order methods} \label{sec:non_conventional}

In the previous section, we achieved acceleration with first-order methods for specific problem classes. However, the lower bound from \cref{cor: lb strongly monotone} still prevents us from doing so for the larger problem classes for smooth and strongly monotone games. To bypass this limitation, we can consider going \textit{beyond} first-order methods. In this section, we consider two different approaches. The first one is adaptive acceleration, which is a \textit{non-oblivious} first-order method. The second is consensus optimization, an inversion-free second order method.

\subsection{Adaptive acceleration}

In previous sections, we considered shapes whose optimal polynomial is known. This optimal polynomial lead to an optimal first-order method. However, when the shape is \textit{unknown}, we cannot use better methods than EG with an appropriate stepsize.

Recent work in optimization analysed adaptive algorithms, such as \textit{Anderson Acceleration} \citep{walker2011anderson}, that are adaptive to the problem constants. They can be seen as an automatic way to find the optimal combination of the previous iterates. Recent works on Anderson Acceleration extended the theory for non-quadratic minimization, by using regularisation \citep{scieur2016regularized} (RNA method). The theory has also been extended to ``non symmetric operators" \citep{bollapragada2018nonlinear}, and this setting fits perfectly the one of games, as $\J_F(\omega^*)$ is not symmetric.

Anderson acceleration and its extension RNA are similar to quasi-Newton \citep{fang2009two}, but remains first-order methods. Even if they find the optimal first-order method (for linear $F$), they cannot beat a lower bound similar to \cref{cor: lb strongly monotone}, when the number of iterations is smaller than the dimension of the problem. The next section shows how to use \textit{cheap} second-order information to improve the convergence rate.

\subsection{Momentum consensus optimization}

CO \citep{meschederNumericsGANs2017a} iterates as follow:
\begin{align*}
    \omega_{t+1} & = \omega_t - \alpha\big(F(\omega_t) + \tau \J_F^T(\omega)F(\omega) \big).
\end{align*}
Albeit being a second-order method, each iteration requires only one Jacobian-vector multiplication. This operation can be computed efficiently by modern machine learning frameworks, with automatic differentiation and back-propagation. 
For instance, for neural networks, the computation time of this product or the gradient is comparable. 
Moreover, unlike Newton's method, CO does \textit{not} require a matrix inversion.

Though CO is a second-order method, its analysis can still be reduced to our framework by considering the following transformation of the initial operator  $F(\omega)$,
\begin{equation} \label{eq:consensus_field}
    F^{\text{cons.}}(\omega) = F(\omega) + \tau \nabla \left(\tfrac{1}{2}\|F\|^2\right)(\omega)\,.  
\end{equation}

Though the eigenvalues of $\J_{F^{\text{cons.}}}$ are not purely real in general, their imaginary to real part ratio can be controlled by \citet[Lem.~9]{meschederNumericsGANs2017a} as,
\begin{equation*}
    \textstyle \max_{\lambda \in \Sp \J_{F^{\text{cons.}}}(\omega^*)} \tfrac{|\Im \lambda|}{|\Re \lambda|} = O\left(\tfrac{1}{\tau}\right)\,.
\end{equation*}
Therefore, if $\tau$ increases, these eigenvalues move closer to the real axis and can be included in a thin ellipse as described by \cref{subsection: ellipses}. We then show that, if $\tau$ is large enough, this ellipse can be chosen thin enough to fall into the accelerated regime of \cref{prop: perturbed acceleration} and therefore, adding momentum achieves acceleration.
\begin{prop}[{restate=[name=]}]\label{prop: consensus}
Let $\sigma_i$ be the singular values of $\J_{F}(\omega^*)$. Assume that
\varvspace{-1mm}
\begin{equation*}
    \textstyle \gamma \leq \sigma_i \leq L, \quad %
    \tau = \frac{L}{\gamma^2}.
\varvspace{-2mm}
\end{equation*}
There exists $\alpha,\beta$, s.t., momentum applied to $F^{\text{cons.}}$,
\begin{align*}
    \omega_{t+1} 
    & = \omega_t - \alpha F^{\text{cons.}}(\omega_t) + \beta(\omega_t - \omega_{t-1})
    \varvspace{-4mm}
\end{align*}
converges locally at a rate $O\Big(\big(1 - c\tfrac{\gamma}{L} + M\tfrac{\gamma^2}{L^2}\big)^t\Big)$ where $c = \sqrt 2 - 1$ and $M$ is an absolute constant.
\end{prop}
Hence, adding momentum to CO yields an accelerated rate. The assumption on the Jacobian encompasses both strongly monotone and bilinear games. On these two classes of problems, CO is at least as fast as any oblivious first-order method as its rate roughly matches the lower bounds of \cref{prop: kappa lower bound} and \ref{prop: kappa bilinear}. 

Note that, choosing $\tau$ of this order is what is done by \citet{abernethyLastiterateConvergenceRates2019a} %
for (non-accelerated) CO. They claim that this point of view -- seeing consensus as a perturbation of gradient descent on $\frac{1}{2}\|F\|^2$ -- is justified by practice as in the experiments of \citet{meschederNumericsGANs2017a}, $\tau$ is set to 10.

\begin{comment}
Note that, as this rate is similar to the one that can be obtained with the standard momentum method applied to minimizing the objective $\frac{1}{2}\|F\|^2$, one could wonder what is the advantage of Consensus Optimization over the latter.
Actually a plain gradient descent on $\frac{1}{2}\|F\|^2$ does not behave well in practice unlike Consensus Optimization \citep{meschederNumericsGANs2017a} and can be attracted to unstable equilibria in non-monotone landscapes \citep{letcherDifferentiableGameMechanics2019}.
\end{comment}

%

\section{Conclusion}
This paper shows that a spectral perspective is fundamental to understand the conditioning of games. The latter is indeed linked to the geometric properties of the distribution of the spectrum of its Jacobian. In the light of this perspective, we demonstrate how several gradient-based methods transform the spectral shape of a game to achieve accelerated convergence when combined with Polyak momentum.
Our main tool throughout this paper was the flexible and convenient class of ellipses; we left as future work the study of more intricate shapes, which -- ideally -- would fit the distribution of the eigenvalues of applications of challenging machine learning problems such as GANs.

\subsubsection*{Acknowledgments} %
\label{par:paragraph_name}
This research was partially supported by the Canada CIFAR AI Chair Program, the Canada Excellence Research Chair in ``Data Science for Realtime Decision-making", by the NSERC Discovery Grants RGPIN-2017-06936 and RGPIN-2019-06512, the FRQNT new researcher program 2019-NC-257943, by a Borealis AI fellowship, a Google Focused Research award and a startup grant by IVADO. Simon Lacoste-Julien is a CIFAR Associate Fellow in the Learning in Machines \& Brains program.

\bibliographystyle{plainnat}
\bibliography{references}
\newpage
\appendix
\onecolumn

\section{Linear algebra results}
\begin{thm}[Spectral Mapping Theorem]\label{thm: spectral mapping theorem}
Let $A \in \mathbb{C}^{d\times d}$ and $P$ be a polynomial. Then,
\begin{equation}
\Sp P(A) = \{P(\lambda)\ |\ \lambda \in \Sp A\}\,.    
\end{equation}
\end{thm}
See for instance \citet[Theorem 4, p.~66 ]{laxLinearAlgebraIts2007} for a proof.
\section{Proofs of general lemmas}\label{section: proof of general lemmas}
\propObliviousFirstOrder*
\begin{proof}
The fact that any first-order method as defined by \cref{def: first-order method} satisfies such relations is immediate. The converse can be shown by induction. Assume that $(\omega_t)_t$ are generated by the rule of \cref{prop:oblivious_first_order}.
For $t = 0$, the condition of \cref{def: first-order method} is trivial. Assume that for all $k \leq t$, $w_k \in \omega_0 + \Span\{F(\omega_0),\dots,F(\omega_{k-1})\}$. Then, 
\begin{align*}
\omega_{t+1} &=\sum_{k = 0}^t \alpha_k^{(t)} F(\omega_k) + \beta_k^{(t)} \omega_k\\
&= \omega_0 + \sum_{k = 0}^t \alpha_k^{(t)} F(\omega_k) + \beta_k^{(t)} (\omega_k - \omega_0)\\
&\in \omega_0 + \Span\{F(\omega_0),\dots,F(\omega_{t})\}\,.
\end{align*}
\end{proof}
\lemmaFirstOrderMethod*
\begin{proof}
We use \cref{prop:oblivious_first_order} to prove this result by induction. For $t = 0$, the statement holds. Now assume that for all $k \leq t$, $\omega_{k} - \omega^* = p_{k}(A)(\omega_0 - \omega^*)$ with $p_{t'}$ a real polynomial of degree at most $t'$ such that $p_{t'}(0) = 1$ (and which depends only on the coefficients of \cref{prop:oblivious_first_order}). Note that if $F(\omega^*) = 0$, then as F is linear, we can rewrite $F$ as $F(\omega) =  A(\omega - \omega^*)$.
Then, by \cref{prop:oblivious_first_order}, as $\sum_{k=0}^t \beta_k^{(t)} = 1$,
\begin{align*}
\omega_{t+1}  - \omega^* &=\sum_{k = 0}^t \alpha_k^{(t)} F(\omega_k) + \beta_k^{(t)} (\omega_k - \omega^*)\\
&=\sum_{k = 0}^t \alpha_k^{(t)} A(\omega_k - \omega^*) + \beta_k^{(t)} (\omega_k - \omega^*)\\
&=\sum_{k = 0}^t \alpha_k^{(t)} A p_k(A) (\omega_0 - \omega^*) + \beta_k^{(t)} p_k(A)(\omega_0 - \omega^*)\\
&= p_{t+1}(A)(\omega_0 - \omega^*)\,,
\end{align*}
where $p_{t+1}(X) = \sum_{k = 0}^t \alpha_k^{(t)} X p_k(X)  + \beta_k^{(t)} p_k(X)$, which is a real polynomial of degree at most $t+1$. Then $p_{t+1}(0) = \sum_{k = 0}^t \beta_k^{(t)} p_k(0) = 1$, which concludes the proof.
\end{proof}
\lemmaAugmentedOperatorSpectralRate*
\begin{proof}
This is a direct application of \cref{thm: local convergence} to $V_{augm}:\R^d \times \R^d \rightarrow \R^d \times \R^d$.
\end{proof}
\propKappaLowerBound*
\textbf{Note.} If we work in $\mathbb{C}^d$ this proposition is immediate. However, as we constrain ourselves to real vectors and matrices, this is slightly more difficult. This is why we need the matrix representation of complex numbers which is described in the following lemma.
\begin{lemma}\label{lemma: representation of complex numbers}
Define, for $z \in \C$, the real $2 \times 2$ matrix $C(z) = \begin{pmatrix}
\Re z &-\Im z\\
\Im z & \Re z\\
\end{pmatrix}$. Then,
\begin{enumerate}[label=(\roman*)., font=\itshape]
    \item The spectrum of $C(z)$ is $\Sp C(z) = \{z, \bar z\}$.
    \item $C$ is $\R$-linear,
    \[
    \forall z,z' \in \C,\, a, a' \in \R,\quad C(az + a'z') = aC(z) + a'C(z')\,.
    \]
    \item $C$ is a multiplicative group homorphism,
    \[
    \forall z,z' \in \C,\quad C(zz') = C(z)C(z')\,.
    \]
    
\end{enumerate}
\end{lemma}

We now show a small lemma which will be useful to construct matrices in $\mathcal{M}_K$.

\begin{lemma}\label{lemma: emptiness of problem class}

Let $K \subset \C$ be a subset of $\C$ symmetric w.r.t.~the real axis, and such that $\mathcal{M}_K \neq \emptyset$. If $d \geq 3$, then, 
\[
\{ A \in \mathbb{R}^{d-2} :\; \Sp(A) \subset K\} \neq \emptyset\,.
\]
\end{lemma}
\begin{proof}
We consider two cases, depending on the parity of $d$.
\begin{itemize}
    \item Assume that $d$ is odd. We show that this implies that $K$ intersects the real axis. Let $M$ be a matrix in $\mathcal{M}_K$ as it is non-empty by assumption. Then, as the dimension $d$ is odd, $M$ has at least one real eigenvalue, i.e.~$\Sp M \cap \R \neq \emptyset$. Hence, $K \cap \R \neq \emptyset$ and let $\nu K \cap \R$ be such an element. Then, the matrix $\diag(\nu, \dots,\nu)\in \R^{(d-2)\times (d-2)}$, which is the square diagonal matrix of size $d-2$ with only $\nu$ on its diagonal, belongs to $\{ A \in \mathbb{R}^{d-2} :\; \Sp(A) \subset K\}$ which proves the claim.
    \item Assume that $d$ is even.  As $\mathcal{M}_K \neq \emptyset$, $K \neq \emptyset$ and so take $\lambda \in K$. As $K$ is assumed to be symmetric w.r.t.~the real axis, $\bar \lambda$ belongs to $K$ too. As $d$ is even, we can then define the matrix $M = \diag(C(\lambda), \dots, C(\lambda)) \in \R^{(d-2) \times (d-2)}$ which a real block-diagonal matrix. Its spectrum is $\Sp M = \Sp C(\lambda) = \{\lambda, \bar \lambda\} \subset K$ so it proves the claim.
\end{itemize}
\end{proof}
\begin{proof}
We write this proof with $\omega^* = 0$ without loss of generality. Consider an oblivious first-order method, given by its sequence of polynomials $p_t \in \P_t$, $t\geq 0$. Fix $t \geq 0$ and take $\lambda \in \argmax_{z \in K}|p_t(z)|$. 

We now build $A \in \mathcal M _K$ which has $\lambda$ as an eigenvalue. First assume that $d \geq 3$. Then, by \cref{lemma: emptiness of problem class}, there exists $M \in \R^{(d-2) \times (d-2)}$ such that $\Sp M \subset K$. Now construct $A$ as,
\[
A = \left(
\begin{array}{c|c}
     C(\lambda) & 0_{2 \times (d -2)} \\ \hline
     0_{(d - 2) \times 2} & M \\ 
\end{array}
\right)\,.
\]
If $d = 2$, simply take $A = C(\lambda)$.

As $A$ is block-diagonal, $\Sp A = \Sp C(z) \cup \Sp M = \{\lambda, \bar \lambda\} \cup \Sp M$. By definition $\Sp M \subset K$ and, as $\lambda \in K$ and $K$ is symmetric w.r.t.~the real axis, $\{\lambda, \bar \lambda\} \subset K$ too. Hence $\Sp A \subset K$ and so $A \in \mathcal{M}_K$.

We now look at the iterates of the method applied to the vector field $x \mapsto Ax$ to prove the claim. As $\omega^* = 0$, $\|\omega_t - \omega^*\|_2 = \|\omega_t\|_2 = \|p_t(A) \omega_0\|_2$.

To explicit $p_t(A)$, we need to compute $p_t(C(\lambda))$. But, as $p_t$ is a real polynomial, by \cref{lemma: representation of complex numbers}, we have $p_t(C(\lambda)) = C(p_t(\lambda))$.
Hence, 
\[
p_t(A) = \left(
\begin{array}{c|c}
     C(p_t(\lambda)) & 0_{2 \times (d -2)} \\
     \hline
     0_{(d - 2) \times 2} & p_t(M) \\ 
\end{array}
\right)\,.
\]
Now take $\omega_0 = \begin{pmatrix} 1 &0& \dots & 0\end{pmatrix}^T$. Then $\|p_t(A)\omega_0\|^2 = (\Re (p_t(\lambda)))^2 + (\Im (p_t(\lambda)))^2 = |p_t(\lambda)|^2$ and so $\|p_t(A)\omega_0\| = \max_{z \in K}|p_t(z)| \|\omega_0\| \geq \acf(K)^t \|\omega_0\|$, which yields the result.

\end{proof}

\begin{rmk}[{Definition of the asymptotic convergence factor in the matrix iteration literature}]\label{rmk: definition acf litterature}
The original definitions of the asymptotic convergence factor for linear systems iterations and in particular the one in \citet{nevanlinnaConvergenceIterationsLinear1993} (which is called \emph{optimal reduction factor} in their work), are actually different from the one we presented here. Indeed, the authors work with complex numbers all along so they consider methods with potentially non-real coefficients. Hence, they define the asymptotic convergence factor as,
\begin{equation}
    \acf(K)' = \inf_{t>0} \min_{q_t\in \mathcal{Q}_t}  \max_{\lambda \in K}  \sqrt[t]{|q_t(\lambda)|} \,,
\end{equation}
where $\mathcal{Q}_t$ is the set of \emph{complex} polynomials $q_t$ of degree at most $t$ such that $q_t(0) = 1$.
However, for infinite $K$ which are symmetric w.r.t.~the real axis, these two definitions, the one with the complex polynomials and the one with the real polynomials, coincide, as, for all $t \geq 0$,
\begin{equation}
\min_{q_t\in \mathcal{Q}_t}  \max_{\lambda \in K}  {|q_t(\lambda)|} = \min_{p_t\in \mathcal{P}_t}  \max_{\lambda \in K}  {|p_t(\lambda)|}\,.    
\end{equation}
This is a consequence of the uniqueness of such minimizers, see \citet[Cor.~3.5.4]{nevanlinnaConvergenceIterationsLinear1993}.
\end{rmk}

The following lemma justifies our choice of spectral problem class for strongly monotone and Lipschitz vector fields.
\begin{lemma}\label{lemma: strongly monotone problem class}
Let $F:\R^d \rightarrow \R^d$ be a continuously differentiable vector field $\mu$-strongly monotone and $L$-Lipschitz. Then, for all $\omega \in \R^d$,
\begin{equation}
    \mu \leq \Re \lambda\,,\quad |\lambda| \leq L\,,\quad \forall \lambda \in \Sp \J_F(\omega)\,.
\end{equation}
\end{lemma}
\begin{proof}
Fix $\omega \in \R^d$. 
The first step is standard, see \citet[Prop.~2.3.2]{facchineiFiniteDimensionalVariationalInequalities2003} for instance.
From the strong monotonicity and the Lipschitz assumptions, for any $\omega' \in \R^d$,
$$
(\omega - \omega')^T(F(\omega) - F(\omega')) \geq \mu \|\omega - \omega'\|^2\,,\quad \|F(\omega) - F(\omega')\| \leq L\|\omega - \omega'\|\,.
$$
Take $u \in \R^d$.
Letting $\omega' = \omega + t u$, dividing by, respectively, $t^2$ and $t$, and letting $t$ goes to zero yields, 
$$u^T\J_F(\omega)u \geq \mu\|u\|^2\,,\quad \|\J_F(\omega)u\|\leq L\|u\|\,.$$
From the second inequality, we get that $\|\J_F(\omega)\| \leq L$ and so the magnitudes of the eigenvalues of $\J_F(\omega)$ are bounded by $L$. From the first one, we get that $\mathcal{H}(\J_F(\omega)) \coloneqq \frac{\J_F(\omega) + \J_F(\omega)^T}{2} \succeq \mu I_d$.
Now, for $\lambda \in \Sp \J_F(\omega)$, and $v \in \C^d$ associated eigenvector with $\|v\| = 1$, then $\J_F(\omega)v = \lambda v$ and so $\lambda = \bar v^T \J_F(\omega) v$. In particular $\Re \lambda = \frac{\lambda + \bar \lambda}{2} = \bar v^T \mathcal{H}\left(\J_F(\omega)\right)v \geq \mu \|v\|^2 = \mu$, which yields the result.
\end{proof}
\begin{lemma}[{\citet[\S 6.2]{eiermannStudySemiiterativeMethods1985}; \citet[Example 3.8.2]{nevanlinnaConvergenceIterationsLinear1993}}]\label{lemma: optimality gradient method}
Let $K = \{z \in \C:|z - c| \leq r\}$ with $c > r > 0$.
Then, for all $t \geq 0$, the polynomial

\[
    p_t^*(z) = \left(1 - \frac{z}{c}\right)^t\,.
\]
is optimal, i.e.,
\[
p_t^* \in \argmin_{p_t \in \P_t} \max_{z \in K}|p_t(z)|\,.
\]
and so $\acf(K) = \frac{r}{c}$.
Moreover, the gradient method with step-size $1/c$ is optimal for $K$: for any vector field $F$ such that $\Sp \J_F(\omega^*) \subset K$, the gradient operator defined by
\begin{equation}
\omega_{t+1} = V_{\text{grad}}(\omega_t) = \omega_t - \eta F(\omega_t)\,,
\end{equation}
satisfy, for $\eta = \frac{1}{c}$,
\begin{equation}
\rho(\J_{V_{\text{grad}}}(\omega^*)) \leq \frac{r}{c}\,.
\end{equation}
\end{lemma}
This result is only briefly discussed in the references above and as consequence of broader theories. For completeness and simplicity we give a simpler proof using Rouché's theorem. We recall a simplified version of this theorem, see \citet[Thm.~10.10]{bakComplexAnalysis2010} for a proof.
\begin{thm}[Rouché]
Let $f$ and $g$ be analytic functions, and $D = \{z \in \C\, |\, |z -z_c| < R\}$ for $z_C \in \C$ and $R > 0$. If for all $z \in \partial D$ the boundary of $D$ it holds that $|f(z)| > |g(z)|$, then the number of zeroes of $f - g$ inside $D$ (counted with multiplicity) is the same as the number of zeroes of $f$ inside $D$.  
\end{thm}
\begin{proof}[Proof of \cref{lemma: optimality gradient method}]
Let $p_t^*(z) = \left(1 - \frac{z}{c}\right)^t$ which belongs to $\P_t$. For the sake of contradiction assume that $p_t^*$ is not optimal, i.e.~there exists $q_t \in \P_t$ different from $p_t^*$ such that
\[
\max_{z \in K}|p_t^*(z)| > \max_{z \in K}|q_t(z)|\,,
\]
where $K$ was defined in the statement as $K = \{z \in \C:|z - c| \leq r\}$ with $c > r > 0$. Observe that $|p_t^*|$ reaches its maximum $\left(\frac{r}{c}\right)^t$ on $K$ everywhere on the boundary of $K$,
\[
\max_{z \in K}|p_t^*(z)| = \left(\frac{r}{c}\right)^t = |p_t^*(z_b)|\quad \forall z_b \in \partial K\,.
\]
\end{proof}
Hence, for all $z_b \in \partial K$,
\[
|q_t(z_b)| \leq \max_{z \in K}|q_t(z)| < \max_{z \in K}|p_t^*(z)| = |p_t^*(z_b)|\,.
\]
Therefore, as $q_t$ and $p_t^*$ are polynomials and in particular analytic, we can apply Rouché's theorem with $D = \interior K$ and this yields that the number of zeroes of $p_t^* - q_t$ in $\interior K$ is the same as the number of zeroes of $p_t^*$ in $\interior K$.  On the one hand, $c$, which belongs to the interior of $K$, is a zero of multiplicity $t$ of $p_t^*$. On the other hand, as $q_t$ and $p_t^*$ are in $\P_t$, they satisfy $p_t^*(0) = 1 = q_t(0)$ and so $(p_t^* - q_t)(0) = 0$. However, as $c > r$, $0$ is not in $K$. So, as $(p_t^* - q_t)$ is of degree at most $t$, it can have at most $t -1$ remaining zeroes (counted with multiplicity) in $ \interior K$. This contradicts the conclusion of Rouché's theorem that $p_t^* - q_t$ must have exaclty $t$ zeroes inside $K$. Therefore, there exists no such $q_t$ and so $p_t^* \in \argmin_{p_t \in \P_t} \max_{z \in K}|p_t(z)|$.

Moreover, this implies that $\min_{p_t \in \P_t} \max_{z \in K}|p_t(z)| = \left(\frac{r}{c}\right)^t$ and so that $\acf(K) = \frac{r}{c}$.

What is left to check is the bound $\rho(\J_V(\omega^*)) \leq \frac{r}{c}$. Recall that $V_{grad}(\omega) = \omega - \eta F(\omega)$ and so that $\J_{V_{\text{grad}}}(\omega) = I_d - \eta \J_F(\omega)$. By the spectral mapping theorem (\cref{thm: spectral mapping theorem}), \[
\Sp \J_{V_{\text{grad}}}(\omega) = \{ 1 -\eta \lambda\, |\, \lambda \in \Sp \J_F(\omega)\}\,.
\]
Letting $\eta = \frac{1}{c}$ and using that $\Sp \J_F(\omega) \subset K$ yields the result.

\clearpage
\section{Acceleration related proofs}
\subsection{Bilinear games}\label{subsection: app bilinear games}
We recall Polyak's theorem.
\begin{thm}[{\citet[Thm.~9]{polyakMethodsSpeedingConvergence1964}}]\label{thm: polyak momentum}
 Let $0 < \mu < L$. 
 Define Polyak's Heavy-ball method as
 \begin{equation}
 \omega_{t+1} = V^{\text{Polyak}}_{\alpha, \beta}(\omega_{t}, \omega_{t-1}) = \omega_t - \alpha F(\omega_t) + \beta(\omega_t - \omega_{t-1})\,.    
  \end{equation}
  For $\alpha = \frac{4}{(\sqrt{\mu} + \sqrt{L})^2}$ and $\beta = \left(\frac{\sqrt L - \sqrt \mu}{\sqrt L + \sqrt \mu}\right)^2$ and for any vector field $F$ such that $\Sp \nabla F(\omega^*) \subset [\mu, L]$, then
  \begin{equation}
      \rho(\nabla V^{\text{Polyak}}_{\alpha, \beta}(\omega^*, \omega^*)) \leq \acf([\mu, L])=\frac{\sqrt L - \sqrt \mu}{\sqrt L + \sqrt \mu}\,.
  \end{equation}
 \end{thm}
 
In this subsection we first prove the following result.
\propBilinearOpt*
\begin{proof}
This proposition follows from \cref{thm: polyak momentum}. Indeed, the Jacobian of $V^{real}$ at $\omega^*$ is,
\begin{align*}
\J_{V^{real}}(\omega^*) &= \frac{1}{\eta}(\J_F(\omega^*) (\Id - \eta\J_F(\omega^*)) - \J_F(\omega^*))\\
&= -\J_F(\omega^*)^2\,,
\end{align*}
where we used $F(\omega^*) = \omega^*$.
Now, we can deduce the spectrum of $\J_{V^{real}}(\omega^*)$ from the one of $\J_F(\omega^*)$ using the spectral mapping theorem \cref{thm: spectral mapping theorem},
\begin{align*}
\Sp \J_{V^{real}}(\omega^*) &= \{-\lambda^2\ |\ \lambda \in \Sp \J_F(\omega^*)\}\\
&\subset \{-\lambda^2\ |\ \lambda \in \pm [ia, ib]\}\\
&\subset [a^2, b^2]\,.
\end{align*}
We can now apply Polyak's momentum method to $V^{real}$ and we get the desired bound on the spectral radius by \cref{thm: polyak momentum}, with $\alpha = \frac{4}{(a + b)^2}$ and $\beta = \left(\frac{b - a}{b + a}\right)^2$.
\end{proof}

We now prove the following lemma, in order to use \cref{prop: bilinear opt} on bilinear games. Note that as $A$ is square, $\sigma_{min}(A)^2$ and $\sigma_{max}(A)^2$ actually correspond to, respectively, the smallest and the largest eigenvalues of $AA^T$

\begin{lemma}\label{lemma: spectrum bilinear game}
Consider the bilinear game
\begin{equation}\label{eq: bilinear game}
    \min_{x \in \R^{m}} \max_{y \in \R^{m}} x^TA y + b^T x + c^Ty\,.
\end{equation}
Let $F : \R^{2m} \rightarrow \R^{2m}$ be the associated vector field. Then,
\begin{equation}
\Sp \nabla F(\omega^*) \subset [i\sigma_{min}(A), i\sigma_{max}(A)] \cup [- i\sigma_{min}(A), -i\sigma_{max}(A)]\,.
\end{equation}
\end{lemma}
\begin{proof}
We have $F(\omega) = \begin{pmatrix} Ay + b\\
- A^Tx - c
\end{pmatrix}$ and so
$$
\nabla F(\omega) = \begin{pmatrix}
0_{m \times m} & A\\
 - A^T & 0_{m \times m}
\end{pmatrix}\,.
$$
We compute the characteristic polynomial of $\nabla F(\omega)$ using the bloc determinant formula, which can be found in \citet[Section 0.3]{zhangSchurComplementIts2005}, as $A^T$ and $I_m$ commute,
\begin{align*}
    \det(XI_{2m} - A) &= \begin{vmatrix}
XI_{m} &  - A\\
 A^T & XI_{m}
\end{vmatrix}\\
&= \det(X^2I_m + AA^T)\,.
\end{align*}
Hence, $\Sp \nabla F(\omega) = \{ \pm i \lambda\, |\, \lambda^2 \in \Sp AA^T\}$ which gives the result.
\end{proof}

We now prove the optimality of this method. For this we rely on \citet[Thm.~6]{eiermannHybridSemiIterativeMethods1989}, that we state below for completeness.
\begin{thm}[{\citet[Thm.~6]{eiermannHybridSemiIterativeMethods1989}}] \label{thm:eiermann}
    Let $\Omega \subset \C$ be a compact set such that $0 \notin \Omega$, $\Omega$ has no isolated points and $\C\cup{\infty}\setminus \Omega$ is of finite connectivity. Consider $t_n$ polynomial of degree $n$ such that $t_n(0) = 0$ and define $\tilde \Omega = t_n(\Omega)$.
    If, $t_n^{-1}(\tilde\Omega) = \Omega$, then we have,
    \[
        \acf(\Omega) = \acf(\tilde \Omega)^{1/n}
    \]
\end{thm}
\propKappaBilinear*
\begin{proof}
We use \cref{thm:eiermann} with $\Omega = \pm [ia, ib]$, $t_2(X) = -X^2$ and $\tilde \Omega = [a^2, b^2]$.
We get,
\begin{align*}
\acf(\underbrace{\pm [ia, ib]}_{=K}) = \acf([a^2, b^2])^{1/2} = \sqrt\frac{b - a}{b + a}\,.
\end{align*}
\end{proof}
\subsection{Ellipses}\label{subsection: app ellipses}
Define, for $a, b, c \geq 0$, the ellipse
\begin{equation}
    E(a, b, c) = \left\{\lambda \in \C: \frac{(\Re \lambda - c)^2}{a^2} + \frac{(\Im \lambda)^2}{b^2} \leq 1\right\}\,.
\end{equation}
As mentioned earlier, we work with shapes symmetric w.r.t.~the real axis and in $\C_+$ (the set of complex number with non-negative  real part). So the ellipses we consider have their center
on the positive real axis and we we will require below that $0 \notin E(a, b, c)$.
Ellipses have been studied in the context of matrix iteration, due to their flexibility and their link to the momentum method. The next theorem can be considered as a summary and reinterpretation of the literature on the subject. The way to obtain it from the literature, and a partial proof, are deferred to the \Cref{section: ellipses}.
\begin{thm}\label{thm: ellipses}
Let $a, b\geq 0$, $c > 0$, $(a, b) \neq 0$, such that $0 \notin E(a, b, c)$. Then, if $\rho(a, b, c) < 1$, 
\begin{equation}
\acf(E(a, b, c)) = \rho(a, b, c) 
\end{equation}
where
\begin{equation}
    \rho(a, b, c) = 
    \begin{cases}
        \frac{a}{c} &\text{ if } a = b\\
        \frac{c - \sqrt{b^2 + c^2 - a^2}}{a -b} &\text{ otherwise}
    \end{cases}
\end{equation}
Assume that $F$ is any vector  field $F: \R^d \rightarrow \R^d$ that satisfies $\Sp \nabla F(\omega^*) \subset E(a, b, c)$. There exists $\alpha(a, b, c) > 0$, $\beta(a, b, c) \in (-1, 1]$, whose sign%
is the same as $a - b$ such that, for the momentum operator $V : \R^d\times\R^d \rightarrow \R^d\times\R^d$,
\begin{equation}
    V(\omega,\omega') = (\omega - \alpha F(\omega) + \beta(\omega - \omega'), \omega'),
\end{equation}
we have
\begin{equation}
    \rho(\J_V(\omega^*, \omega^*)) \leq \rho(a, b, c)\,.
\end{equation}
More exactly the corresponding parameters are given by,
    \begin{equation}
    \beta(a, b, c) = 
    \begin{cases}
    0 &\text{ if } a=b\\
    2c\frac{c - \sqrt{c^2 + b^2 -  a^2}}{a^2 - b^2} - 1 &\text{ otherwise,}
    \end{cases}
    \quad\quad\quad \alpha(a, b, c) = \frac{1+\beta}{c} = 
        \begin{cases}
    \frac{1}{c} &\text{ if } a=b\\
    2\frac{c - \sqrt{c^2 + b^2 -  a^2}}{a^2 - b^2} &\text{ otherwise,}
    \end{cases}
    \end{equation}
    and $\beta(a, b, c)$ can be written $\beta = \chi(a, b, c)(a - b)$ with $\chi(a, b, c) > 0$.
\end{thm}
\begin{rmk}[On the sign of the momentum parameter]
As briefly mentioned in the theorem, and detailed in \cref{prop: parametrizations ellipse}, the optimal momentum parameter $\beta(a, b, c)$ has the same sign as $a - b$, i.e., more exactly, there exists $\chi(a, b, c) > 0$ such that $\beta(a, b, c) = \chi(a, b, c)(a -b)$. Hence the sign of the optimal $\beta$ has a nice geometric interpretation, which answers some of the questions left open by \citet{gidelNegativeMomentumImproved2018b}. 
\begin{itemize}
    \item In the case where $a>b$, or equivalently $\beta > 0$, the ellipse is more elongated in the direction of the real axis. The extreme case is a segment on the real line that corresponds to strongly convex optimization.
    \item In the case where $a<b$,  or equivalently $\beta < 0$, the ellipse is more elongated in the direction of the imaginary axis.
    \item Finally, when $a=b$ we have a disk instead of an ellipse. For such shape, we have no momentum, which means that gradient descent is optimal as seen in \cref{subsection: lower bound strongly monotone games}.
\end{itemize}
\end{rmk}
\subsection{Perturbed acceleration}
In this subsection, we prove \cref{prop: perturbed acceleration}. Note that the constants in the $\bigO(.)$ are absolute.
\begin{prop}\label{prop: perturbed acceleration 2}
Define $\epsilon(\mu, L)$ as $\frac{\epsilon(\mu, L)}{L} = \left(\frac{\mu}{L}\right)^{\theta}$ with $\theta > 0$ and $a \wedge b = \min(a,b)$. Then, when $\frac{\mu}{L} \rightarrow 0$, 
\begin{equation}
\acf(K_{\epsilon}) = \begin{cases}
 1 - 2\sqrt{\frac{\mu}{L}} + \bigO\left(\left(\frac{\mu}{L}\right)^{\theta \wedge 1}\right),& if \; \theta > \frac{1}{2}\\
 1 - 2(\sqrt 2 - 1)\sqrt{\frac{\mu}{L}} + \bigO\left(\left(\frac{\mu}{L}\right)\right),& if \; \theta = \frac{1}{2}\\
1 - \left(\frac{\mu}{L}\right)^{1-\theta} + \bigO\left(\left(\frac{\mu}{L}\right)^{1 \wedge (2 - 3\theta)}\right),& if \; \theta < \frac{1}{2}.
    \end{cases} \notag
\end{equation}
Moreover, the momentum method is optimal for $K_\epsilon$. This means
that with $\alpha\left(\frac{L-\mu}{2}, \epsilon, \frac{L + \mu}{2}\right) > 0$ and $\beta\left(\frac{L-\mu}{2}, \epsilon, \frac{L + \mu}{2}\right) > 0$ (as defined in \cref{thm: ellipses}),  if $\Sp \J_F(\omega^*) \subset K_\epsilon$, then,
$\rho(\J_{V^{\text{Polyak}}}(\omega^*, \omega^*)) \leq \acf(K_\epsilon)$.
\end{prop}
where $K_\epsilon$ is the ellipse defined by,
\begin{equation}
K_\epsilon = \Bigg\{z \in \C: \bigg(\frac{\Re z - \frac{\mu + L}{2}}{\frac{L-\mu}{2}}\bigg)^2 + \left(\frac{\Im z}{\epsilon}\right)^2 \leq 1\Bigg\}\,.
\end{equation}
\begin{proof}
A direct application of \cref{thm: ellipses} using $K_\epsilon$ gives $\acf(K_\epsilon) = \rho(\frac{L - \mu}{2}, \epsilon(\mu, L), \frac{L+\mu}{2})$.\\
We now study $\rho(\frac{L - \mu}{2}, \epsilon(\mu, L), \frac{L+\mu}{2})$. First note that
\begin{align*}
\rho(a, b, c) = 1 - \frac{a - b - c + \sqrt{b ^ 2 + c^2 - a^2}}{a - b}\,, 
\end{align*}
and so
\begin{align*}
1 - \rho(a, b, c) = \frac{\sqrt{b ^ 2 + c^2 - a^2} + a - b - c}{a - b}\,.
\end{align*}
We now replace $a$, $b$ and $c$ by their expressions (and multiply the denominator and numerator by 2).
\begin{align*}
1 - \acf(K_\epsilon)
&= \frac{\sqrt{4\epsilon^2 + (L+\mu)^2 - (L - \mu)^2} + (L - \mu)- 2\epsilon - (L + \mu)}{L - \mu - 2\epsilon}\\
&= 2\frac{\sqrt{\epsilon^2 + \mu L} - 2\epsilon - 2\mu}{L - \mu - 2\epsilon}\,.\\
\end{align*}
Define $t = \frac{\mu}{L}$, then $\frac{\epsilon}{L} = t^\theta$. We are now interested in studying the behaviour of $1 - \acf(K_\epsilon)$ when $t$ goes to zero.
\begin{align*}
1 - \acf(K_\epsilon)
&= 2\frac{\sqrt{t^{2\theta} + t} - 2t^\theta - 2t}{1 - t - 2t^\theta}\\
&= 2(\sqrt{t^{2\theta} + t} - 2t^\theta - 2t)(1 + t + 2t^\theta + \bgO{t^{2(\theta \wedge 1)}}\,,\\
&= 2\left(\sqrt{t^{2\theta} + t} - 2t^\theta - 2t\right)\left(1 + \bgO{t^{\theta \wedge 1}}\right)\,,\\
\end{align*}
where $a\wedge b$ denotes $\min(a, b)$.
\begin{itemize}
    \item If $\theta = \frac{1}{2}$. This is the smallest $\theta$ with which acceleration happens.
    \begin{align*}
    1 - \acf(K_\epsilon)
    &= 2\left(\sqrt{2t} - 2\sqrt t - 2t\right)\left(1 + \bgO{\sqrt t}\right)\\
    &= 2\left(\sqrt{2} - 1\right)\sqrt t\left(1 + \bgO{\sqrt t}\right)\left(1 + \bgO{\sqrt t}\right)\\
    &= 2\left(\sqrt{2} - 1\right)\sqrt t + \bgO{t}\,.
    \end{align*}
    \item If $\theta > \frac{1}{2}$. This regime is "better" than the previous one, i.e.,~the perturbation is even smaller so we get a similar asymptotic behavior, up to an improved constant.
    \begin{align*}
    1 - \acf(K_\epsilon)
    &= 2\left(\sqrt{t^{2\theta} + t} - 2t^\theta - 2t\right)\left(1 + \bgO{t^{\theta \wedge 1}}\right)\\
    &= 2\sqrt t \left(\sqrt{t^{2\theta - 1} + 1} - 2t^{\theta - 1/2} - 2\sqrt t\right)\left(1 + \bgO{t^{\theta \wedge 1}}\right)\\
    &= 2\sqrt t \left(1 + \bgO{t^{2\theta - 1}} - 2t^{\theta - 1/2} - 2\sqrt t\right)\left(1 + \bgO{t^{\theta \wedge 1}}\right)\\
    &= 2\sqrt t + \bgO{t^{\theta \wedge 1}}\,.\\
    \end{align*}
    \item If $\theta < \frac{1}{2}$. In this regime, terms in $t^\theta$ are limiting as they are bigger than $\sqrt t$, so we do not get the rate in $\sqrt{t}$.
    \begin{align*}
    1 - \acf(K_\epsilon)&= 2\left(\sqrt{t^{2\theta} + t} - t^\theta - t\right)\left(1 + \bgO{t^\theta}\right)\\
    &= 2\left(t^\theta\sqrt{1 + t^{1 - 2\theta}} - t^\theta - t\right)\left(1 + \bgO{t^\theta}\right)\\
    &= 2\left(t^\theta\left(1 + \frac{1}{2}t^{1 - 2\theta} + \bigO\left(t^{2 - 4\theta}\right)\right) - t^\theta - t\right)\left(1 +\bgO{t^\theta}\right)\\
    &= 2t^{1 - \theta}\left(\frac{1}{2} + \bigO\left(t^{1 - 2\theta}\right) - t^\theta\right)\left(1 +\bgO{t^\theta}\right)\\
    &= t^{1 - \theta} + \bgO{t^{1 \wedge \left(2 - 3\theta\right)}}\,.\\
    \end{align*}
\end{itemize}
\end{proof}

Before going further we need to introduce this technical lemma.

\begin{lemma}\label{lemma: small polynomial lemma}
For any real $m \geq 2$, we have
\begin{equation}
\forall x \in [0, 1]:\quad \frac{(1 - x)^2}{(1 - \frac{x}{m})^2} + x \leq 1\,. 
\end{equation}
\end{lemma}
\begin{proof}
Indeed,
\begin{align}
(1 - x^2) + x\left(1 - \frac{x}{m}\right)^2 - \left(1 - \frac{x}{m}\right)^2 &= 1 - 2x + x^2 + x - 2\frac{x^2}{m} + \frac{x^3}{m^2} - 1 + 2\frac{x}{m} - \frac{x^2}{m^2}\\
&= \left(\frac{2}{m}-1\right)x +\left(1 - \frac{2}{m}\right)x^2 + \frac{x^3 - x^2}{m^2}\\
&= \left(1 - \frac{2}{m}\right)x(x - 1) + \frac{x^3 - x^2}{m^2}\\
\end{align}
Then, as $m \geq 2$, $1 - \frac{2}{m} \geq 0$, and so, since $0 \leq x \leq 1$, $\left(1 - \frac{2}{m}\right)x(x - 1) \leq 0$. As $0 \leq x \leq 1$, we also have $\frac{x^3 - x^2}{m^2} \leq 0$ which concludes the proof.
\end{proof}

\subsection{Acceleration of extragradient on bilinear games}\label{subsection: app extragradient}
We now prove that we can accelerate EG on bilinear games.
\propExtragradient*
This proposition is a consequence of \cref{lemma: local convergence momentum} combinded with the following result.
\begin{prop}\label{prop: extragradient 2}
Consider the vector field $F$, where $\Sp \J_F(\omega^*) \subset [ia,ib]\cup[-ia, -ib]$ for $0<a<b$. There exists $\alpha, \beta, \eta > 0$ such that, the operator defined by 
\begin{align*}
    \omega_{t+1} 
    & = V^{Polyak+e-g}(\omega_t, \omega_{t-1}), \\
    & = \omega_t - \alpha F(\omega_t - \eta F(\omega_t)) + \beta (\omega_t - \omega_{t-1})\,,
\end{align*}
satisfies, with $c = \sqrt{2}-1$, and absolute constants in the $\bigO(.)$,

\begin{equation}
\textstyle \rho(\J_{V^{Polyak+e-g}}(\omega^*, \omega^*)) \leq 1 - c{\frac{a}{b}} + O(\frac{a^2}{b^2})\,.
\end{equation}
More precisely, the parameters are chosen as:
    \begin{align*}
\eta = \frac{b}{a\sqrt{2b^2 - \frac{a^2}{2}}}\quad \alpha = \alpha\left(\eta\left(b^2 - \frac{a^2}{2}\right),b,\eta b^2\right) \quad \beta = \beta\left(\eta\left(b^2 - \frac{a^2}{2}\right),b,\eta b^2\right) \,,
    \end{align*}
with where the functions $\alpha(\cdot)$ and $\beta(\cdot)$ are the ones defined in \cref{thm: ellipses}.
\end{prop}
Note that this proposition actually requires $\alpha$ and $\eta$ to be tuned separately and they are actually very different. They actually differ by a factor $\frac{b}{a}$: $\eta$ is roughly proportional to $\frac{1}{a}$ while $\alpha$ behaves like $\frac{b^2}{a}$.

\begin{proof}
    Consider $F^{\text{e-g}}(\omega) = F(\omega - \eta F(\omega))$. Then, for $\omega^*$ such that $F(\omega^*) = 0$, we have,
    \begin{equation*}
        \J_{F^{\text{e-g}}}(\omega^*) = \J_F(\omega^*) - \eta \J_F(\omega^*)^2\,.  
    \end{equation*}
    Hence, by \cref{thm: spectral mapping theorem},
    \begin{equation*}
        \Sp \nabla\J_{F^{\text{e-g}}}(\omega^*) \subset \{z - \eta z^2\, |\, z \in \pm [ia, ib]\} = \{i\lambda +\eta \lambda^2\, |\, \lambda \in \pm [a, b]\}\,.
    \end{equation*}
    As we want to apply \cref{prop: perturbed acceleration 2}, we now look for $\epsilon, \bar\mu, \bar L$ such that the ellipsis 
    \begin{equation}
        K(\epsilon, \bar\mu, \bar L) = \Bigg\{z \in \C: \bigg(\frac{\Re z - \frac{\bar\mu + \bar L}{2}}{\frac{\bar L-\bar\mu}{2}}\bigg)^2 + \left(\frac{\Im z}{\epsilon}\right)^2 \leq 1\Bigg\} \,
    \end{equation}
    contains $\Sp \nabla\J_{F^{\text{e-g}}}(\omega^*)$. For this we will choose $\epsilon, \bar\mu, \bar L$ such that $$\{i\lambda +\eta \lambda^2\, |\, \lambda \in \pm [a, b]\} \subset K(\epsilon, \bar\mu, \bar L)\,,$$
    which is equivalent to for all $\lambda \in [a, b]$,
    \begin{equation}\label{eq: proof accelerated eg}
      \bigg(\frac{\eta \lambda^2 - \frac{\bar\mu + \bar L}{2}}{\frac{\bar L-\bar\mu}{2}}\bigg)^2 + \left(\frac{\lambda }{\epsilon}\right)^2 \leq 1\,.
    \end{equation}
    Note that the left-hand side is convex in $\lambda^2$ so that we only need to check this inequality for the limit values $\lambda = a$ and $\lambda = b$.  
    Hence, now we have reduced the problem to that of looking for $\bar \mu$, $\bar L$ and $\epsilon$ such that $\lambda = a$ and $\lambda = b$ satisfy \cref{eq: proof accelerated eg}. This is equivalent to look for an ellipse $K(\epsilon, \bar \mu, \bar L)$ that contains $ib + \eta b^2$ and $ia + \eta a^2$. 
    
    We now construct this ellipsis explicitly. As we want to apply \cref{prop: perturbed acceleration 2}, we want $\epsilon$ as small as possible. So we start with $ib + \eta b^2$ as it is the one with the largest imaginary part, compared to 
    $ia + \eta a^2$. 
    We choose the center of the ellipse -- which must lie on the real axis -- such that it is placed at the same abscisse as $ib + \eta b^2$, i.e.~the same real part.
  So we define $\frac{\bar \mu + \bar L}{2} = \eta b^2$. We need another condition to fix $\bar \mu$ and $\bar L$. To make sure that $ia + \eta a^2$ is also in the ellipsis, we need to choose $\bar\mu$ small enough. Define $\bar \mu = \frac{\eta a^2}{m}$ with $m > 0$ to be chosen later. This fixes the value of $\bar L$ as,
    \[
        \bar L = 2\eta b^2 - \bar \mu = 2\eta b^2 - \eta \frac{a^2}{m}\,.
    \]
    We choose $\epsilon$ so that $ib +\eta b^2$ is in the ellipsis, and as we chose the center to be $\frac{\bar \mu + \bar L}{2} = \eta b^2$, we define $\epsilon = b$. This way $ib + \eta b^2 \in K(\epsilon, \bar \mu, \bar L)$. We must now check that $ia + \eta a^2 \in K(\epsilon, \bar \mu, \bar L)$. For this we check that $\lambda = a$ satisfies \cref{eq: proof accelerated eg},
    \begin{align}
        &\bigg(\frac{\eta a^2 - \frac{\bar\mu + \bar L}{2}}{\frac{\bar L-\bar\mu}{2}}\bigg)^2 + \left(\frac{a}{\epsilon}\right)^2\\
        &= \left(\frac{\eta a^2 - \eta b^2}{\eta b^2 - \eta \frac{a^2}{m}}\right)^2 + \left(\frac{a}{b}\right)^2\\
        &= \frac{(1 - x)^2}{(1 - \frac{x}{m})^2} + x\,,\\
    \end{align}
    where $x = \frac{a^2}{b^2} \in [0, 1]$. By \cref{lemma: small polynomial lemma}, for $m = 2$, this quantity is smaller than one and so $ia + \eta a^2 \in K(\epsilon, \bar \mu, \bar L)$.
    Hence, $K(\epsilon, \bar \mu, \bar L)$ contains $\Sp \nabla\J_{F^{\text{e-g}}}(\omega^*)$.
    
    Before we apply \cref{prop: perturbed acceleration 2}, we need to make sure that we are in the accelerated regime, that is to say $\epsilon$ is small enough compared to $\bar \mu$ and $\bar L$. Fortunately, we have not chosen $\eta$ yet. So we define it so that we reach the accelerated regime,
    \begin{align*}
    \frac{\epsilon}{\bar L} = \sqrt{\frac{\bar \mu}{\bar L}} \iff \epsilon = \sqrt{\bar \mu \bar L} \iff \eta = \frac{b}{a\sqrt{2b^2 - \frac{a^2}{m}}}\,.
    \end{align*}
    We now apply \cref{prop: perturbed acceleration 2}. As $\frac{\bar \mu}{\bar L}$ goes to zero, 
    $\acf (K(\epsilon, \bar \mu, \bar L)) = 1 - 2(\sqrt 2 - 1)\sqrt{\frac{\bar \mu}{\bar L}} + \bgO{\frac{\bar \mu}{\bar L}}$.

    Now note that $\frac{\bar \mu}{\bar L} = \frac{a^2}{2mb^2 - a^2}$, so if $\frac{a}{b} \rightarrow 0$ then $\frac{\mu}{L} \rightarrow 0$. Moreover, $\frac{\bar \mu}{\bar L} = \frac{a^2}{2mb^2} + \bgO{\frac{a^4}{b^4}}$.
    Hence, as we chose $m = 2$,
    \begin{equation*}
     \acf (K(\epsilon, \bar \mu, \bar L)) =  1 - (\sqrt 2 - 1)\frac{a}{b} + \bgO{\frac{a^2}{b^2}}\,.
    \end{equation*}
    The parameters of the momentum method applied to $F^{\text{e-g}}$ are then chosen according to \cref{prop: perturbed acceleration} so that we reach this rate locally.
\end{proof}

In the previous proposition we showed that EG can be accelerated but that it requires a careful choice of the parameters $\alpha$, $\beta$, $\eta$. The following lemma describes the general behavior of these quantities when the condition number worsens. 
\begin{lemma}\label{lemma: asymptotic behavior of step sizes}
In the context of the previous proposition, \cref{prop: extragradient 2}, it holds, when $\frac{a}{b} \rightarrow 0$,
\begin{align*}
    \eta &= \frac{1}{a}\left(\frac{1}{\sqrt 2} + \bgO{\frac{a^2}{b^2}}\right)\\
    \alpha &= \frac{a}{b^2}\left(\small{2\sqrt 2} + \bgO{\frac{a}{b}}\right)\\
    \beta &= 1 - 2\sqrt{3} \frac{a}{b} + \bgO{\frac{a^2}{b^2}} \,.
\end{align*}
\end{lemma}
\begin{proof}
For compactness, denote $t = \frac{a}{b} > 0$. So we study the asymptotic behavior of $\alpha$ and $\eta$ when $t$ goes to 0. By definition of $\eta$ we have,
\[
\eta = \frac{b}{a\sqrt{2b^2 - \frac{a^2}{2}}} = \frac{1}{\sqrt 2 a \sqrt{1 - \frac{t^2}{4}}} =\frac{1}{\sqrt 2 a}\left(1 + \frac{t^2}{8} + \bgO{t^4}\right)\,,
\]
which gives the first claim as $\frac{t^2}{8} + \bgO{t^4} = \bgO{t^2}$.

Before moving to the second claim, let us consider some consequences of this asymptotic expansion of $\eta$. Indeed, we have, 
$$
\eta b^2 = \eta \frac{a^2}{t^2} = \frac{a}{\sqrt 2 t^2}\left(1 + \frac{t^2}{8} + \bgO{t^4}\right)\,
$$
and, using the expansion above,
$$
\eta \left(b^2 - \frac{a^2}{2}\right) = \eta b^2\left(1 -\frac{t^2}{2}\right) = \frac{a}{\sqrt 2 t^2}\left(1 + \frac{t^2}{8} + \bgO{t^4}\right)\left(1 - \frac{t^2}{2}\right) = \frac{a}{\sqrt 2 t^2}\left(1 - \frac{3t^2}{8} + \bgO{t^4}\right)\,.
$$
We can now study the behavior of $\alpha = \alpha\left(\eta\left(b^2 - \frac{a^2}{2}\right),b,\eta b^2\right)$. Recall that $\alpha(.,.,.)$ is the function defined in \cref{thm: ellipses} and that its definition depends on whether its first and second arguments are equal. However, suing the expansion above, the ratio of its first and second arguments is $$\frac{\eta\left(b^2 - \frac{a^2}{2}\right)}{b} = \frac{1}{\sqrt 2 t}\left(1 - \frac{3t^2}{8} + \bgO{t^4}\right)\,,$$ which diverges to infinity as $t$ goes to zero. Hence, when $t$ is close enough to zero, $\eta\left(b^2 - \frac{a^2}{2}\right)$ and $b$ are different and so we have,
$$\alpha = \alpha\left(\eta\left(b^2 - \frac{a^2}{2}\right),b,\eta b^2\right) = 2\frac{\eta b^2 - \sqrt{(\eta b^2)^2 + b^2 - \left(\eta\left(b^2 - \frac{a^2}{2}\right)\right)^2}}{\left(\eta\left(b^2 - \frac{a^2}{2}\right)\right)^2 - b^2}\,.$$
We first consider the term under the square root,
\begin{align*}
\sqrt{(\eta b^2)^2 + b^2 - \left(\eta\left(b^2 - \frac{a^2}{2}\right)\right)^2}
&= \sqrt{
\frac{a^2}{2t^4}\left(1 + \frac{t^2}{8} + \bgO{t^4}\right)^2+\frac{a^2}{t^2} - \frac{a^2}{2 t^4}\left(1 - \frac{3t^2}{8} + \bgO{t^4}\right)^2
}\\
&= \sqrt{
\frac{a^2}{2t^4}\left(1 + \frac{t^2}{4} + \bgO{t^4}\right)+\frac{a^2}{t^2} - \frac{a^2}{2 t^4}\left(1 - \frac{3t^2}{4} + \bgO{t^4}\right)
}\\
&= \sqrt{
\frac{3}{2}\frac{a^2}{t^2} + a^2 \times \bgO{1}
}\\
&= \sqrt{\frac{3}{2}}\frac{a}{t}\sqrt{
1 + \bgO{t^2}
}\\
&= \sqrt{\frac{3}{2}}\frac{a}{t}\left(1 + \bgO{t^2}\right)\,.\\
\end{align*}
We can now give the expansion of $\alpha$ when $t = \frac{a}{b}$ goes to zero,
\begin{align*}
\alpha
&= 2\frac{\eta b^2 - \sqrt{(\eta b^2)^2 + b^2 - \left(\eta\left(b^2 - \frac{a^2}{2}\right)\right)^2}}{\left(\eta\left(b^2 - \frac{a^2}{2}\right)\right)^2 - b^2}\\
&= 2\frac{\frac{a}{\sqrt 2 t^2}\left(1 + \bgO{t^2}\right) - \sqrt{\frac{3}{2}}\frac{a}{t}\left(1 + \bgO{t^2}\right)}{\frac{a^2}{2 t^4}\left(1 + \bgO{t^2}\right) - \frac{a^2}{t^2}}\\
&= 2\frac{\frac{1}{\sqrt 2}\left(1 + \bgO{t^2}\right) - \sqrt{\frac{3}{2}}{t}\left(1 + \bgO{t^2}\right)}
{\frac{a}{2 t^2}\left(1 + \bgO{t^2}\right) - {a}}\,
\end{align*}
by multiplying both the numerator and the denominator by $\frac{t^2}{a}$. Then, if we factorize $\frac{a}{t^2}$ in the denominator, we get,
\begin{align*}
\alpha    
&= 2\frac{t^2}{a}\frac{\frac{1}{\sqrt 2}\left(1 + \bgO{t^2}\right) - \sqrt{\frac{3}{2}}{t}\left(1 + \bgO{t^2}\right)}
{\frac{1}{2}\left(1 + \bgO{t^2}\right) - {t^2}}\,\\
&= 2\frac{t^2}{a}\frac{\frac{1}{\sqrt 2} -  \sqrt{\frac{3}{2}}{t} + \bgO{t^2}}
{\frac{1}{2} + \bgO{t^2}}\,\\
&= 2\sqrt 2 \frac{t^2}{a}(1 - \sqrt{3}{t} + \bgO{t^2})\,,\\
\end{align*}
which yields the result for $\alpha$.

Recall that, from the definition of $\alpha(\cdot)$ and $\beta(\cdot)$ in \cref{thm: ellipses}, we have,
$$\beta = \eta b^2 \alpha - 1\,,$$
and so, when $t = \frac{a}{b}$ goes to zero,
\begin{align*}
\beta
&= \frac{a}{\sqrt 2 t^2}(1 + \bgO{t^2})\times 2\sqrt 2 \frac{t^2}{a}(1-  \sqrt{3}{t} + \bgO{t^2}) - 1 = 1 - 2 \sqrt{3} t + \bgO{t^2}\,.
\end{align*}

\end{proof}

\subsection{Consensus optimization and momentum}
\begin{figure}
    \centering
    \setcounter{mu}{3}
    \setcounter{L}{42}
    \newcounter{invq}
    \setcounter{invq}{2}
    \newcommand{\coellipsecolor}{orange}
    \newcommand{\cotrapezoidcolor}{blue}
    
    \begin{tikzpicture}[scale=1]
        \begin{axis}[
            grid=major,
            axis equal image,
            yticklabel={
            $\pgfmathprintnumber{\tick}i$
            },
            width=12cm,
        	xmin=0,   xmax={\value{L} + 3},
        	ymin={-\value{L}/3},   ymax={\value{L}/3},
        	]
           \draw [domain=-180:180, thick, draw=\coellipsecolor, fill=\coellipsecolor!75!white, fill opacity=0.3, samples=65] plot (axis cs: {(\value{L} + \value{mu})/2 + cos(\x) * (\value{L} - \value{mu})/2}, {sin(\x) * sqrt(\value{mu}*\value{L})});
            \node[anchor=south east] (mu) at (axis cs: {\value{mu}}, 0) {$\bar \mu$};
           \draw [fill=black] (axis cs: {\value{mu}}, 0) circle (1pt);
           \node [anchor=south west] (L) at (axis cs: {\value{L}}, 0) {$\bar L$};
           \draw [fill=black] (axis cs: {\value{L}}, 0) circle (1pt);
           \node [anchor= west] (epsilon) at (axis cs:{0.5 * (\value{mu} + \value{L})},{0.5 * sqrt(\value{mu}*\value{L})}) {$\epsilon$};
           \node [anchor=west] (center) at (axis cs:{0.5 * (\value{mu} + \value{L})},0) {$\frac{\bar \mu+\bar L}{2}$};
           \draw [fill=black] (axis cs:{0.5 * (\value{mu} + \value{L})},0) circle (1pt);
           \node[] (realminplus) at (axis cs: {2 * \value{mu}}, {2 * \value{mu}/\value{invq}}) {};
           \node[] (realminminus) at (axis cs: {2 * \value{mu}}, {-2 * \value{mu}/\value{invq}}) {};
           \node  (realmaxplus) at (axis cs: {(\value{mu} + \value{L})/2}, {(\value{mu} + \value{L})/(2 *\value{invq})}) {};
           \node[] (realmaxminus) at (axis cs: {(\value{mu} + \value{L})/2}, {-(\value{mu} + \value{L})/(2 *\value{invq})}) {};
           \draw[thick, draw=\cotrapezoidcolor] (realminplus.center) -- (realminminus.center);
           \draw[thick, draw=\cotrapezoidcolor] (realmaxplus.center) -- (realmaxminus.center);
           \draw[thick, draw=\cotrapezoidcolor, name path=plus] (realmaxplus.center) -- (realminplus.center);
           \draw[thick, draw=\cotrapezoidcolor, name path=minus] (realmaxminus.center) -- (realminminus.center);
    \addplot [
        thick,
        color=\cotrapezoidcolor,
        fill=\cotrapezoidcolor!50!white,
        fill opacity=0.5,
    ]
    fill between[
        of= plus and minus,
    ];
        \end{axis}
    \end{tikzpicture}
    \caption{Illustration of the proof of \cref{prop: consensus 2}. \cref{lemma: consensus optimization spectrum} guarantees that the spectrum of $\J_{F^{\text{cons.}}}$ is located inside a trapezoid (in blue). We then find a suitable ellipse of the form of \cref{prop: perturbed acceleration 2} (in orange) which contains it.}
    \label{fig: illustration proof consensus}
\end{figure}

The general idea behind the proof of \cref{prop: consensus 2} is illustrated by \cref{fig: illustration proof consensus}. Using the following lemma, we first prove that the eigenvalues of the consensus optimization operator $F^{cons.}$ are contained in a trapezoid (in blue on the figure). Then, we find a suitable ellipse of the form of \cref{prop: perturbed acceleration 2} (in orange) such that the trapezoid, thus the spectrum of $\J_{F^{cons.}}$ as shown by \cref{lemma: consensus optimization spectrum}, fits inside.

First we need to refine \citet[Lem.~9]{meschederNumericsGANs2017a}.
\begin{lemma}\label{lemma: consensus optimization spectrum}
Let $A \in \R^{d \times d}$ be a square matrix. Let $\sigma_i$ be the singular values of $A$. Assume that
\begin{equation*}
    \textstyle \gamma \leq \sigma_i \leq L, \quad \mathcal{H}(A)\coloneqq \frac{A + A^T}{2} \succeq \mu I_d\,,
\end{equation*}
with $\gamma > 0$ and $\mu \geq 0$. Then, for $\tau > 0$ such that $\tau\gamma^2 \geq \mu(1 + 2\tau \mu)$,
\begin{equation*}
\max \left\{\left.\frac{|\Im \lambda|}{|\Re\lambda|}\, \right|\, \lambda \in \Sp(A + \tau A^T A)\right\} \leq \frac{\gamma}{\mu + \tau\gamma^2}\,.
\end{equation*}
Moreover, for $\lambda \in \Sp(A + \tau A^T A)$, we have $\mu + \tau \gamma^2 \leq \Re \lambda \leq L + \tau L^2$.
\end{lemma}
\begin{proof}
In this proof, for $M$ a real matrix, $\mathcal{H}(M) = \frac{M + M^T}{2}$ its Hermitian part and $\mathcal{S}(M = (M - M^T)/2$ its skew-symmetric part.

Let $B = A + \tau A^TA$. Let $\lambda \in \Sp B$ and let $v \in \C^d$ its associated eigenvector with $\|v\|=1$.
Then 
\begin{equation}\label{eq: proof lemma spectrum consensus}
\Re \lambda = \frac{\lambda + \bar \lambda}{2} = \bar v ^T \mathcal{H}(B) v = \bar v ^T \mathcal{H}(A) v + \tau \|Av\|^2 \geq \mu \|v\|^2 + \tau \|Av\|^2 = \mu + \tau \|Av\|^2\,,    
\end{equation}
by the assumption on $\mathcal H(A)$. We now deal with the imaginary part,
\[
    \Im \lambda = \frac{\lambda - \bar \lambda}{2i} = \frac{1}{i}\bar v^T \mathcal{S}(A) v\,.
\]
 However, this quantity is hard to bound. Thus, we rewrite it using $\bar v^T A \bar v$ and $\bar v^T \mathcal H(A) v$. We have that $\bar v^T \mathcal H(A) v$ and $\frac{1}{i}\bar v^T \mathcal{S}(A) v$ are real and correspond respectively to the real and imaginary parts of $\bar v^T A v$. Hence,
\[
    (\Im \lambda)^2 = \left(\Im \bar v^T A v)\right)^2 = |\bar v^T A v|^2 - (\Re \bar v^T A v)^2\leq |\bar v^T A v|^2\,.
\]
Using Cauchy-Schwarz inequality we get,
\[
    (\Im \lambda)^2 \leq \|Av\|^2\,.
\]
Finally,
\begin{equation*}
    \frac{|\Im \lambda|}{|\Re \lambda|} \leq \frac{\|Av\|}{\mu + \tau \|Av\|^2} = \varphi(\|Av\|)\,.
\end{equation*}
with $\varphi : x \mapsto \frac{x}{\mu + \tau x^2}$ . Using the derivative $\varphi'$, we have that $\varphi$ is non-decreasing before $x = \sqrt{\frac{\mu}{\tau}}$ and non-increasing after. Note that $\|Av\| \geq \sigma_{min}(A)\geq \gamma$. Hence, if $\tau\gamma^2 \geq \mu$, then $\varphi(\|Av\|) \leq \varphi(\gamma) =  \frac{\gamma}{\mu + \tau\gamma^2}$ which concludes the proof of the first part of the lemma.

Now, take $\lambda \in \Sp(A + \tau A^T A)$. Then the inequality $\mu + \tau \gamma^2 \leq \Re \lambda$ comes from \cref{eq: proof lemma spectrum consensus}. The other one is 
\[
    \Re \lambda \leq |\lambda| = \|Bv\| \leq \|Av\| + \tau \|Av\|^2 \leq L + \tau L^2\,.
\]
\end{proof}
We can now proceed to show \cref{prop: consensus} by proving the more detailed proposition below.
\begin{prop}[{restate=[name=]}]\label{prop: consensus 2}
    Let $\sigma_i$ be the singular values and eigenvalues of $J_{F}(\omega^*)$. Assume that
    \begin{equation*}
        \textstyle \gamma \leq \sigma_i \leq L, \quad \frac{\J_F(\omega^*) + \J_F(\omega^*)}{2} \succeq \mu I_d\,.
    \end{equation*}
    Define $F^{\text{cons.}}(\omega) = F(\omega) + \tau \nabla (\frac{1}{2}\|F\|^2)(\omega)$ with $\tau > 0$ and consider the momentum method applied to $F^{\text{cons.}}$,
    \begin{align*}
        \omega_{t+1} & = V^{\text{mom+cons.}}(\omega_t, \omega_{t-1})\\
        & = \omega_t - \alpha F^{\text{cons.}}(\omega_t) + \beta(\omega_t - \omega_{t-1}).
    \end{align*}
    If $\tau \gamma^2 \geq \mu$ and 
    \begin{equation}\label{eq: app consensus tau condition}
        \frac{\gamma}{\mu + \tau\gamma^2} \leq \sqrt\frac{3}{2}\sqrt{\frac{\mu + \tau\gamma^2}{L + \tau L^2}}\,,
    \end{equation}
    Then, one can choose $\alpha > 0$ and $\beta > 0$ such that,
    \begin{equation}
    \rho(\J_{V^{\text{mom+cons.}}}(\omega^*, \omega^*))) \leq 1 - c\sqrt{\frac{\mu + \tau\gamma^2}{L + \tau L^2}} + \bgO{\frac{\mu + \tau\gamma^2}{L + \tau L^2}}\,.
    \end{equation}
More precisely, the parameters are given by,
\begin{align}
\alpha &= \alpha\left(L + \tau L^2 - \frac{\mu+\tau \gamma^2}{2}, \frac{1}{2}\sqrt{\mu + \tau \gamma^2}\sqrt{4(L + \tau L^2) - (\mu + \tau \gamma^2)}, L + \tau L^2\right)\\
\beta &= \beta\left(L + \tau L^2 - \frac{\mu+\tau \gamma^2}{2}, \frac{1}{2}\sqrt{\mu + \tau \gamma^2}\sqrt{4(L + \tau L^2) - (\mu + \tau \gamma^2)}, L + \tau L^2\right)\,,
\end{align}
where the functions $\alpha(\cdot)$ and $\beta(\cdot)$ are the ones defined in \cref{thm: ellipses}.

    If $\tau = \frac{L}{\gamma^2}$
    Then, $\rho(\J_{V^{\text{mom+cons.}}}(\omega^*, \omega^*)))$ is bounded by
    \[
        \rho(\J_{V^{\text{mom+cons.}}}(\omega^*, \omega^*)))\leq 1 - (\sqrt 2 - 1)\frac{\gamma}{L} + \bgO{\frac{\gamma^2}{L^2}}\,,
    \]
    where the constants in the $\bigO(.)$ are absolute.
\end{prop}
\begin{proof}
Similarly to the proof of \cref{prop: extragradient}, we want to apply \cref{prop: perturbed acceleration 2}. We now look for $\epsilon, \bar\mu, \bar L$ such that the ellipsis 
\begin{equation}
K(\epsilon, \bar\mu, \bar L) = \Bigg\{z \in \C: \bigg(\frac{\Re z - \frac{\bar\mu + \bar L}{2}}{\frac{\bar L-\bar\mu}{2}}\bigg)^2 + \left(\frac{\Im z}{\epsilon}\right)^2 \leq 1\Bigg\} \,
\end{equation}
contains $\Sp \nabla\J_{F^{\text{cons.}}}(\omega^*)$. 
First we compute $\J_{F^{\text{cons.}}}(\omega^*)$, for $F$ twice differentiable. Note that $F^{\text{cons.}}$ can be written as $F^{\text{cons.}}(\omega) = F(\omega) + \tau \J_F^T(\omega) F(\omega)$, thus
\[
    \J_{F^{\text{cons.}}}(\omega^*) = \J_F(\omega^*) + \tau \J_F^T(\omega) \J_F(\omega)\,.
\]
Note that the derivative of $\J_F$ does not appear as $F(\omega^*) = 0$.
From \cref{lemma: consensus optimization spectrum} with $A = \J_F(\omega^*)$,  we have a control on
\[
    q(\tau) \coloneqq \frac{\gamma}{\mu + \tau \gamma^2} \geq \max \left\{\left.\frac{|\Im \lambda|}{|\Re\lambda|}\, \right|\, \lambda \in \Sp(\J_{F^{\text{cons.}}}(\omega^*))\right\} \geq 0\,.
\]
Using $q(\tau)$ and the bounds on the real parts of \cref{lemma: consensus optimization spectrum}, we have that the spectrum of $\J_{F^{\text{cons.}}}(\omega^*)$ is inside the following shape,
\begin{equation*}
    \Sp \J_{F^{\text{cons.}}}(\omega^*) \subset S(\tau) \coloneqq \{\lambda \in \C\ |\ \mu + \tau \gamma^2 \leq \Re \lambda \leq L + \tau L^2,\, |\Im \lambda| \leq q(\tau) \Re \lambda\}\,.    
\end{equation*}
We now only seek to include $S(\tau)$ in an ellipse $K(\epsilon, \bar \mu, \bar L)$. First, we show that we can focus on two points, i.e.~we prove that if $(1 + iq(\tau))(\mu + \tau\gamma^2)$ and $(1 + iq(\tau))(L + \tau L^2)$ belong to $K(\epsilon, \bar \mu, \bar L)$, then $S(\tau) \subset K(\epsilon, \bar \mu, \bar L)$. 

We have that $S(\tau) \cap \{\Re z \geq 0\}$ is a trapezoid, the convex hull of the four points
\[
    (1 + iq(\tau))(\mu + \tau\gamma^2)\;; \quad (1 - iq(\tau)(\mu + \tau\gamma^2)\;; \quad (1 + iq(\tau))(L + \tau L^2)\;; \quad (1 -iq(\tau))(L + \tau L^2).
\]
As $K(\epsilon, \bar \mu, \bar L)$ is convex, we only need to show that these four points belong to the ellipsis. By horizontal symmetry, we can restrict our analysis to the two points $(1 + iq(\tau))(\mu + \tau\gamma^2)$ and  $(1 + iq(\tau))(L + \tau L^2)$.

Therefore, we focus on choosing a symmetric ellipse $K(\epsilon, \bar \mu, \bar L)$ to which $(1 + iq(\tau))(\mu + \tau\gamma^2)$ and $(1 + iq(\tau))(L + \tau L^2)$ belong. The construction of this ellipse is similar to the one of the proof of \cref{prop: extragradient}. Since $(1 + iq(\tau))(L + \tau L^2)$ is the farthest point  from the real axis, to be able to choose $\epsilon$ as small as possible, we put the center of the ellipse at $(L + \tau L^2)$. This way, we force $\frac{\bar L + \bar \mu}{2} = L + \tau L^2$. 

To make sure $(1 + iq(\tau))(\mu + \tau\gamma^2)$ is also in the ellipsis, we need to choose $\bar\mu$ small enough. Define $\bar \mu = \frac{\mu + \tau \gamma^2}{m}$ with $m \geq 2$. This fixes the value of $\bar L$ as,
\[ 
    \bar L = 2(L + \tau L^2) - \bar \mu = 2(L + \tau L^2) - \frac{\mu + \tau \gamma^2}{m}\,.
\]
We now take $\epsilon > 0$ such that $(1 + iq(\tau))(L + \tau L^2)$ is in the ellipsis. We thus have the condition $\epsilon \geq q(\tau)(L + \tau L^2)$. The precise value of $\epsilon$ will be chosen later.

We must now check that $(1 +iq(\tau))(\mu + \tau \gamma^2) \in K(\epsilon, \bar \mu, \bar L)$. For this we check that this point satisfies the equation of the ellipsis, 
\begin{align}
&\left(\frac{\mu + \tau \gamma^2- \frac{\bar\mu + \bar L}{2}}{\frac{\bar L-\bar\mu}{2}}\right)^2 + q(\tau)^2\left(\frac{\mu + \tau \gamma^2}{\epsilon}\right)^2\\
&=\left(\frac{\mu + \tau \gamma^2- (L + \tau L^2)}{L + \tau L^2 - \frac{\mu + \tau \gamma^2}{m}}\right)^2 + q(\tau)^2\left(\frac{\mu + \tau \gamma^2}{\epsilon}\right)^2\\
&\leq\left(\frac{\mu + \tau \gamma^2- (L + \tau L^2)}{L + \tau L^2 - \frac{\mu + \tau \gamma^2}{m}}\right)^2 + \left(\frac{\mu + \tau \gamma^2}{L + \tau L^2}\right)^2\,,
\end{align}
by the choice of $\epsilon$. Now let $x = \frac{\mu + \tau \gamma^2}{L + \tau L^2} \in [0, 1]$. We have,
\begin{align*}
\left(\frac{\mu + \tau \gamma^2- \frac{\bar\mu + \bar L}{2}}{\frac{\bar L-\bar\mu}{2}}\right)^2 + q(\tau)^2\left(\frac{\mu + \tau \gamma^2}{\epsilon}\right)^2 &\leq \left(\frac{1 - x}{1 - \frac{x}{m}}\right)^2 + x^2\\
&\leq \left(\frac{1 - x}{1 - \frac{x}{m}}\right)^2 + x\\
&\leq 1\,,
\end{align*}
by application of \cref{lemma: small polynomial lemma} as $m \geq 2$.

We now fix $\epsilon$. We want to take $\epsilon = \sqrt{\bar \mu \bar L}$ so we can apply \cref{prop: perturbed acceleration 2}. Thus we need $q(\tau)(L + \tau L^2) \leq \sqrt{\bar \mu \bar L}$. Substituting for $\bar \mu$ and $\bar L$, as $\bar L = 2(L + \tau L^2) - (\mu + \tau \gamma^2)/2 \geq \frac{3}{2}(L + \tau L^2)$, this is implied by $q(\tau) \leq  \sqrt{\frac{3}{2}}\sqrt{\frac{\mu + \tau \gamma^2}{L + \tau L^2}}$. Now, if $\tau \gamma^2 \geq \mu$, we can apply \cref{lemma: consensus optimization spectrum} and obtain the bound $q(\tau) \leq \frac{\mu + \tau \gamma^2}{L + \tau L^2}$.
Hence, if 
\begin{equation}
    \frac{\gamma}{\mu + \tau\gamma^2} \leq \sqrt{\frac{3}{2}}\sqrt{\frac{\mu + \tau\gamma^2}{L + \tau L^2}}\,,
\end{equation}
we can apply \cref{prop: perturbed acceleration 2} with $\epsilon = \sqrt{\bar \mu \bar L}$. Hence, one can choose $\alpha, \beta > 0$ such that,
\begin{align*}
    \rho(\J_{V^{\text{mom+cons.}}}(\omega^*, \omega^*))) &\leq 1 - 2(\sqrt 2 - 1)\sqrt{\frac{\bar \mu}{\bar L}} + \bgO{\frac{\bar \mu}{\bar L}},\\
    &= 1 - 2(\sqrt 2 - 1)\sqrt{\frac{\mu+\tau \gamma^2}{2m(L + \tau L^2) - (\mu + \tau \gamma^2)}} + \bgO{\frac{\mu+\tau \gamma^2}{L + \tau L^2}},\\
    &\leq 1 - (\sqrt 2 - 1)\sqrt{\frac{\mu + \tau \gamma^2}{L + \tau L^2}} + \bgO{\frac{\mu+\tau \gamma^2}{L + \tau L^2}},\\
\end{align*}
as $m = 2$. This yields the first part of the proposition. We now need to find an admissible $\tau$.

Assume $\tau L \geq 1$ and so $L \leq {\tau L^2}$. Then,
\begin{align*}
&\frac{\gamma}{\mu + \tau\gamma^2} \leq \sqrt \frac{3}{2}\sqrt{\frac{\mu + \tau\gamma^2}{L + \tau L^2}},\\
&\impliedby \frac{\gamma}{\mu + \tau\gamma^2} \leq \frac{\sqrt 3}{2}\sqrt{\frac{\mu + \tau\gamma^2}{\tau L^2}},\\
&\iff \frac{\gamma}{\mu + \tau\gamma^2} \leq \frac{\sqrt 3}{2L} \sqrt{\frac{\mu + \tau\gamma^2}{\tau}},\\
&\iff \frac{r}{\mu + \tau\gamma^2} \leq \frac{\sqrt 3}{2L} \sqrt{\frac{\mu + \tau\gamma^2}{\tau \gamma^2}},\\
&\impliedby\frac{1}{\mu + \tau\gamma^2} \leq \frac{\sqrt 3}{2L} \sqrt{\frac{\mu + \tau\gamma^2}{\mu + \tau \gamma^2}},\\  
&\iff \frac{1}{\mu + \tau\gamma^2} \leq \frac{\sqrt 3}{2L}\,.
\end{align*}
After rearranging, we get that the last condition is equivalent to,
\begin{align*}
    \tau \geq \frac{\frac{2}{\sqrt 3}L - \mu }{\gamma^2}\,,
\end{align*}
which is implied by $\tau \geq \frac{L}{\gamma^2}$.

Then, if $\tau \geq \frac{L}{\gamma^2}$, $\tau L \geq 1$ and $\tau \gamma^2 \geq \mu$ and so this condition implies $q(\tau)(L+\tau L^2) \leq \sqrt{\bar \mu \bar L}$, which is what we wanted.

Then, for $\tau = \frac{L}{\gamma^2}$, we have 
\begin{align*}
\frac{\mu + \tau \gamma^2}{L + \tau L^2} &= \frac{\gamma^2}{L^2}\frac{1 + \mu/L}{1 + \gamma^2/L^2}\\
&\geq \frac{\gamma^2}{L^2}\frac{1}{1 + \gamma^2/L^2}\\
&= \frac{\gamma^2}{L^2} + \bgO{\frac{\gamma^4}{L^4}}\\
\end{align*}
and also in particular $\bgO{\frac{\mu + \tau \gamma^2}{L + \tau L^2}} = \bgO{\frac{\gamma^2}{L^2}}$.
Hence, \begin{align*}
    \rho(\J_{V^{\text{mom+cons.}}}(\omega^*, \omega^*)))
    &\leq 1 - (\sqrt 2 - 1)\frac{\gamma}{L} + \bgO{\frac{\gamma^2}{L^2}}\,.
\end{align*}
\end{proof}
\begin{rmk}
Note that, this rate is roughly similar to the one that can be obtained with the standard momentum method applied to minimizing the objective $$f(\omega) = \frac{1}{2}\|F\|^2\,.$$
Indeed, one can check, at a stationary point $\omega^*$, the eigenvalues of of the Hessian of $f$ are in $[\gamma^2, L^2]$ (with the notations of the previous proposition).
So applying \cref{thm: polyak momentum} would yield a local convergence rate of $\bgO{\left(1 - \frac{2\gamma}{L + \gamma}\right)^t}$.

One could then wonder what is the advantage of Consensus Optimization over the latter.
Actually a plain gradient descent on $\frac{1}{2}\|F\|^2$ does not behave well in practice unlike Consensus Optimization \citep{meschederNumericsGANs2017a} and can be attracted to unstable equilibria in non-monotone landscapes \citep{letcherDifferentiableGameMechanics2019}.
\end{rmk}
\begin{rmk}
Though this is not the focus of this paper, similarly to the result of \citet{abernethyLastiterateConvergenceRates2019a} in the non-accelerated case, taking $\tau$ slightly higher, such as $\tau = \frac{2L}{\gamma^2}$, guarantees this same accelerated rate even in non-monotone setting. Indeed, all we need is that $\min_{\lambda \in \Sp \J_F(\omega^*)} \Re \lambda + \tau \gamma^2 > 0$, which is always satisfied by $\tau = \frac{2L}{\gamma^2}$ as the eigenvalues of $\J_F(\omega^*)$ are bounded by $L$.
\end{rmk}

\newpage
\section{Ellipses}\label{section: ellipses}
\subsection{Main results}
We recall the definition of the ellipses which interests us. Define, for $a, b, c \geq 0$, the ellipse:
\begin{equation}
    E(a, b, c) = \left\{\lambda \in \C: \frac{(\Re \lambda - c)^2}{a^2} + \frac{(\Im \lambda)^2}{b^2} \leq 1\right\}\,.
\end{equation}
We adopt the convention that $\frac{0}{0} = 0$ so that for $b = 0$ the ellipse $E(a, b, c)$ degenerates into a real segment.

We now need to define objects related to the momentum method, and in particular its \emph{$\rho$-convergence region}.
For $\alpha, \rho \geq 0$, $\beta \in \R$, define 
\begin{equation}
    S(\alpha, \beta, \rho) = \{\lambda \in \C: \forall z \in \C,\, z^2 - (1 - \alpha \lambda + \beta) z + \beta = 0 \implies |z| \leq \rho\}\,.
\end{equation}
We call it the {$\rho$-convergence region} of the momentum method as it corresponds to the maximal regions of the complex plane where the momentum method converges at speed $\bgO{\rho^t}$ if the operator has its eigenvalues in this zone. This is formalized by the following lemma,
\begin{lemma}[{\citet[II.7]{saulyevIntegrationEquationsParabolic1964}, \citet{polyakMethodsSpeedingConvergence1964}, \citet[Thm.~3]{gidelNegativeMomentumImproved2018b}}]
Denote the momentum operator applied to the vector field $F$ by
\begin{equation}\label{eq: app ellipses momentum operator}
V(\omega,\omega') = (\omega - \alpha F(\omega) + \beta(\omega - \omega'), \omega')
\end{equation}
with $\alpha \geq 0$ step size and $\beta \in \R$ momentum parameter. Then, for any $\rho \geq 0$,
\begin{equation}
\rho(\nabla V(\omega^*, \omega^*)) \leq \rho
\end{equation}
if and only if $\Sp \nabla F(\omega^*) \subset S(\alpha, \beta, \rho)$.
\end{lemma}
For a proof of this lemma in the context of games, see the proof of Thm.~3 of \citet{gidelNegativeMomentumImproved2018b}.

The next is a geometrical characterization of $S(\alpha, \beta, \rho)$: this is an ellipse, which is described in the following lemma.
\begin{lemma}[{\citet[Cor.~6]{niethammerAnalysisOfkstepIterative1983}}]\label{lemma: momentum convergence region}
If $|\beta| > \rho^2$, $S(\alpha, \beta, \rho) = \emptyset$. Otherwise, if $|\beta| \leq \rho^2$ and $\rho > 0$,
\begin{equation}
S(\alpha, \beta, \rho) = \left\{ \lambda \in \C: \frac{(1 - \alpha \Re\lambda + \beta)^2}{(1+\tau)^2} + \frac{(\alpha\Im \lambda)^2}{(1 - \tau)^2} \leq \rho^2 \right\}\,,
\end{equation}
where $\tau = \frac{\beta}{\rho^2}$.
\end{lemma}
As indicated, this lemma is a consequence of the results of \citet{niethammerAnalysisOfkstepIterative1983}, and more exaclty their section \S 6. However, their notations are significantly different from ours. We give a few elements to help the readers translate their results into our setting.
In $\S 6$ of \citet{niethammerAnalysisOfkstepIterative1983}, they study iterative methods of the form,
$$\omega_{t+1} = \mu_0 (\Id - F(\omega_t)) + \mu_1 \omega_t + \mu_2 \omega_{t-1}\,,$$
with $\mu_0+\mu_1+\mu_2=1$. Developing and using this relation, their iteration rule becomes,
$$\omega_{t+1} = \omega_t -\mu_0 F(\omega_t) + \mu_2(\omega_{t-1} - \omega_t)\,.$$
Identifying with \cref{eq: app ellipses momentum operator}, we get that $\alpha = \mu_0$, $\beta = -\mu_2$ and so $\mu_1 = 1 + \beta - \alpha$.

Moreover, what they denote by $S_\eta(p)$, where $p$ is a variable encompassing the parameters $\mu_0$,$\mu_1$ and $\mu_2$,  actually corresponds to $1 - S(\alpha, \beta, \rho)$ with $\alpha$, $\beta$ linked to $\mu_0$,$\mu_1$, $\mu_2$ as described above and $\eta = \frac{1}{\rho}$. Indeed\footnote{This is a standard convention in the linear system theory. They consider $\omega = T\omega + c$ instead of $A\omega + b$ as they use splittings of $A$.}, $S_\eta(p)$ is meant to be compared to the eigenvalues of $I_d - \nabla F(\omega^*)$ instead of $\nabla F(\omega^*)$.
Hence, the center and the semiaxes of the ellipse $1 - S(\alpha, \beta, \gamma)$ are given by $(6.3)$ of \citet{niethammerAnalysisOfkstepIterative1983} and once translated in our notations yield \cref{lemma: momentum convergence region}.
\begin{rmk}
This lemma actually does not require the complex analysis machinery of \citet{niethammerAnalysisOfkstepIterative1983}. This can be proven by hand using this remark on second-order equations.
Let $0 < \rho \leq 1$ and let $z_1, z_2$ denote the two (possibly equal) roots of $X^2 +bX +c$. Then, $$\max(|z_1|,|z_2|) \leq \rho \iff \begin{cases}
&|c| \leq \rho \\
&|b|^2 + |\Delta|\leq 2\left(\rho^2 + \frac{c^2}{\rho^2}\right)\,,\\ 
\end{cases}$$
where $\Delta = b^2 - 4c$ denote the discriminant of the equation.
\end{rmk}
Now, we can introduce one of the main results of \citet{niethammerAnalysisOfkstepIterative1983}. This is an answer to the natural question: what is $\acf(S(\alpha, \beta, \rho))$ ? In particular is it equal to $\rho$? In other words, is momentum optimal w.r.t.~to its convergence sets?
The answer is yes for the momentum method. Note however that this doe snot hold for all stationary methods, this is linked to tricky questions of existence of branch for the roots of some polynomial equations, see
\citet[\S 3.7]{nevanlinnaConvergenceIterationsLinear1993} for a discussion on this.

\begin{prop}[{\citet[Cor.~10]{niethammerAnalysisOfkstepIterative1983}}]\label{prop: acf for momentum convergence region}
Assume $|\beta| \leq \rho^2 < 1$ and $\alpha > 0$, then \begin{equation}
    \acf(S(\alpha, \beta, \rho)) = \rho\,.
\end{equation}
\end{prop}
Hence, momentum is optimal for the sets $S(\alpha, \beta, \rho)$. What is left to show is that the sets $S(\alpha, \beta, \gamma)$ can represent most ellipses $E(a, b, c)$.
\begin{prop}\label{prop: parametrizations ellipse}
Let $a, b \geq 0$, $c>0$, $(a, b) \neq (0, 0)$. There exists $\alpha > 0$, $\rho > 0$, $\beta \in (-1, 1]$, with $|\beta| \leq \rho$ such that $E(a, b, c) = S(\alpha, \beta,\rho)$ if and only if $a^2 \leq b^2 + c^2$. If it is the case, 
\begin{enumerate}
    \item The triple $(\alpha, \beta, \rho)$ satisfying such conditions is unique.
    \item The corresponding $\beta$ can be written $\beta = \chi(a - b)$ with $\chi > 0$.
    \item The corresponding $\rho$ is equal to:
    \begin{equation*}
    \rho = \begin{cases}
    \frac{a}{c} &\text{ if } a = b\\
    \frac{c - \sqrt{b^2 + c^2 - a^2}}{a -b} &\text{ otherwise.}
    \end{cases}
    \end{equation*}
    \item The parameters $\alpha > 0$ and $\beta \in (-1, 1]$ are given by,
    \begin{equation}
    \beta = 
    \begin{cases}
    0 &\text{ if } a=b\\
    2c\frac{c - \sqrt{c^2 + b^2 -  a^2}}{a^2 - b^2} - 1 &\text{ otherwise,}
    \end{cases}
    \quad\quad\quad \alpha = \frac{1+\beta}{c} = 
        \begin{cases}
    \frac{1}{c} &\text{ if } a=b\\
    2\frac{c - \sqrt{c^2 + b^2 -  a^2}}{a^2 - b^2} &\text{ otherwise,}
    \end{cases}
    \end{equation}
\end{enumerate}
\end{prop}
\begin{proof}
Recall these two parametrizations of an ellipse,
\begin{align*}
    E(a, b, c) &= \left\{\lambda \in \C: \frac{(\Re \lambda - c)^2}{a^2} + \frac{(\Im \lambda)^2}{b^2} \leq 1\right\}\\
    S(\alpha, \beta, \rho) &= \left\{ \lambda \in \C: \frac{(1 - \alpha \Re\lambda + \beta)^2}{(1+\tau)^2} + \frac{(\alpha\Im \lambda)^2}{(1 - \tau)^2} \leq \rho^2 \right\}\,,
\end{align*}
where $\tau = \frac{\beta}{\rho^2}$.
Note that $(a, b) \neq (0,0)$, $E(a, b, c)$ is not reduced to a point. So if these ellipses are equal, $\rho > 0$, and we also have $\alpha > 0$
These ellipses are characterised by their centers and their semiaxes so they are equal if and only if,
\begin{equation}\label{eq: proof ellipse equation system}
\begin{cases}
&\frac{1+\beta}{\alpha} = c\\
&\rho + \frac{\beta}{\rho} = \alpha a\\
&\rho - \frac{\beta}{\rho} = \alpha b\\
\end{cases}
\iff
\begin{cases}
&\frac{1+\beta}{\alpha} = c\\
&\rho = \alpha (a + b)\\
&\frac{\beta}{\rho} = \alpha (b - a)\\
\end{cases}
\iff
\begin{cases}
&\frac{1+\beta}{\alpha} = c\\
&\rho = \frac{1 + \beta}{2}(\tilde a + \tilde b)\\
&\frac{\beta}{\rho} = \frac{1 + \beta}{2}(\tilde a - \tilde b)\,,\\
\end{cases}
\end{equation}
where $\tilde a = \frac{a}{c}$ and $\tilde b = \frac{b}{c}$. We further let $\tilde \beta = 1 + \beta$. Then, the last two equations imply the following equation on $\beta$,
\begin{equation}\label{eq: proof ellipse equation beta}
\beta = \frac{(1 + \beta)^2}{4}(\tilde a^2 - \tilde b^2) \iff \tilde \beta - 1 = \frac{\tilde \beta^2}{4}(\tilde a^2 - \tilde b^2)\,.  
\end{equation}
Its discriminant is $\Delta = 1 - (\tilde a ^2 - \tilde b ^2)$, which is non-negative if and only if $b^2 + c^2 \geq a^2$.

Before solving this equation, we briefly discuss when it degenerates into a degree one equation.
 Indeed, if $a = b$, and so $\tilde a = \tilde b$, the unique solution of \cref{eq: proof ellipse equation beta} is $\tilde \beta = 1$ and so $\beta = 0$. Moreover $\rho = \frac{\tilde a + \tilde b}{2} = \frac{a}{c}$. 
 
 We now assume that $\tilde a ^2 - \tilde b^2 \neq 0$. 
 The two solutions of \cref{eq: proof ellipse equation beta} are,
    $$\tilde \beta_{\pm} = 1 + \beta_{\pm} = 2\frac{1 \pm \sqrt{\Delta}}{\tilde a^2 - \tilde b^2}\,.$$
 
 We distinguish three cases.
\begin{itemize}
    \item If $\Delta = 0$. There is only one solution $\tilde \beta = 1 + \beta = 2$ to \cref{eq: proof ellipse equation beta} and so $\beta = 1$.
    \item If $0 < \Delta < 1$ then in particular $\tilde a > \tilde b$.
    As  $0 < \Delta  < 1$, we also have $0 < \tilde a ^2 - \tilde b^2 < 1$.
    This implies that $\tilde \beta_+ > 2(1 + \sqrt \Delta) > 2$ and so $\beta_+ > 1$ which do not satisfy the desired conditions on $\beta$. We show that $\beta_-$ satisfy them instead. As $\Delta < 1$, $\tilde \beta_- > 0$. Moreover, $\sqrt \Delta \geq \Delta$ and so $\tilde \beta_- \leq 2\frac{1 - \Delta}{\tilde a^2 - \tilde b^2} = 2$ and so $\beta_- \in (-1, 1]$.
    \item If $\Delta > 1$ and so $\tilde a < \tilde b$. One has immediately that $\tilde \beta_+ < 0$ which disqualifies $\beta_+$. On the contrary as $\Delta > 1$, $\tilde \beta_- > 0$ and $\tilde \beta_- = 2\frac{\sqrt{1 + \tilde b^2 - \tilde a^2} - 1}{\tilde b^2 - \tilde a^2} \leq 2\frac{1 + \sqrt{\tilde b^2 - \tilde a^2} - 1}{\tilde b^2 - \tilde a^2} = 2$. And so $\beta_- \in (-1, 1]$.
\end{itemize}
Note that the case $\Delta = 1$ is prevented by the assumption $\tilde a ^2 - \tilde b^2 \neq 0$. 

In each of the three cases above we ended up with,
$$\beta = \beta_- = \tilde \beta_{-} - 1 = 2\frac{1 - \sqrt{1 + \tilde b^2 - \tilde a^2}}{\tilde a^2 - \tilde b^2} - 1 = 2c\frac{c - \sqrt{c^2 + b^2 -  a^2}}{a^2 - b^2} - 1\in (-1, 1]\,.$$
Note that the third equation of \cref{eq: proof ellipse equation system} easily gives that $\beta = \chi(a-b)$ with $\chi > 0$.
We now define $\rho$ with the second equation of \cref{eq: proof ellipse equation system},
$$\rho = \frac{1 + \beta}{2}(\tilde a + \tilde b)\,.$$
As $\beta$ satisfy \cref{eq: proof ellipse equation beta}, $\beta$ and $\rho$ also satisfy the third one of \cref{eq: proof ellipse equation system}. $\alpha$ can then be defined by the first equation of \cref{eq: proof ellipse equation system}. Finally note that the fact that $|\beta| \leq \rho^2$ comes from the combination of the second and the third equations of \cref{eq: proof ellipse equation system}.
\end{proof}
 Note that if $0 \notin E(a, b, c)$, then $c^2 > a^2$ and so the hypothesis of the proposition above is satisfied. \cref{thm: ellipses} is now proven by simply combining all the results in this subsection.
\subsection{Proof of optimality of momentum on its convergence zones}
For this proof, we will need a characterization of $\acf(K)$ using Green functions. We will follow the presentation of \citet{nevanlinnaConvergenceIterationsLinear1993}.
\begin{definition}\label{definition: green function}
The Green function (with pole at $\infty$) of a non-empty, connected, unbounded open set $\Omega \subset \C$ is the unique function $g:\Omega \rightarrow \R$ such that:
\begin{enumerate}%
    \item $g$ is harmonic on $\Omega$.
    \item $g(z) = \log |z| + \bigO(1)$ as $|z| \rightarrow \infty$.
    \item $g(z) \xrightarrow[z \rightarrow \zeta]{} 0$ for every $\zeta \in \partial \Omega$.
\end{enumerate}
\end{definition}

For a compact $K \subset \C$, denote by $G_\infty$ the unbounded connected component of $\bar \C \setminus K$. $\acf(K)$ can then be obtained from the Green function of $G_\infty$. This is not our concern here, but note that the Green function of $G_\infty$ is guaranteed to exist if its boundary is sufficiently nice (see for instance \citet{walshInterpolationApproximationRational1935,ransfordPotentialTheoryComplex1995} for a thorough treatment of this classical question).

The following theorem is a deep result in complex analysis, which links the minimization problem over polynomial which defines $\acf$ to the geometric properties of $K$ through its Green function.
\begin{thm}[{\citet[Prop.~3.4.6, Thm.~3.4.9]{nevanlinnaConvergenceIterationsLinear1993}}]\label{thm: link asf green function}
If $G_\infty$ has a Green function $g$ and if $0 \in G_\infty$, \begin{equation}
    \acf(K) = \exp(-g(0))\,.
\end{equation}
\end{thm}
We will also need the following complex analysis lemma about the Joukowsky map, see \citet[Chap.~VI]{nehariConformalMapping1952} for instance. %
\begin{lemma}\label{lemma: joukowski inverse map}
Let $\psi(z) = z + \frac{1}{z}$. Then $\psi: \bar\C\setminus\{z:|z|\leq 1\} \rightarrow \bar\C\setminus [-1, 1]$ is a conformal mapping. Its inverse $\phi$ is characterized by: for any $z_0 \notin [-1, 1]$, $\phi(z_0)$ is the unique solution of 
\begin{equation}
    z^2 - 2zz_0 + 1 = 0
\end{equation}
outside $\{z:|z|\leq 1\}$. Moreover, $\phi(z) = 2z + \bigO(1)$ when $z \rightarrow \infty$.
\end{lemma}

First we begin with a simple lemma about the convergence zones of momentum.

\begin{lemma}\label{lemma: zero not in momentum convergence region}
If $\rho^2 < 1$, $0 \notin S(\alpha, \beta, \rho)$.
\end{lemma}
\begin{proof}
Consider the equation
\begin{equation}
z^2 -(1+\beta)z + \beta = 0\,.    
\end{equation}
Its two roots are $\beta$ and $1$ which yields the result.
\end{proof}

As the boundary of the set plays a special role in the definition of the Green function, we need to have a precise characterization of it. This is done through the the next two lemmas.
\begin{lemma}\label{lemma: interior momentum convergence regions}
If $0 < |\beta| \leq \rho^2$, then
\begin{equation}
\interior(S(\alpha,\beta,\rho)) = \bigcup_{\rho' > 0 : |\beta| < \rho' < \rho} S(\alpha, \beta, \rho')\,.    
\end{equation}
\end{lemma}
\begin{proof}
The functions $x \mapsto x \pm \frac{\beta}{x}$ are increasing positive on $]\sqrt{|\beta|}, +\infty[$. So their square is also increasing.
By \cref{lemma: momentum convergence region}, 
\begin{equation}\label{eq: proof lemma interior momentum convergence regions}
\interior S(\alpha, \beta, \rho) = \left\{ \lambda \in \C: \frac{(1 - \alpha \Re\lambda + \beta)^2}{(1+\tau)^2} + \frac{(\alpha\Im \lambda)^2}{(1 - \tau)^2} < \rho^2 \right\}\,.
\end{equation}
Define for $x > \sqrt{|\beta|}$ the function $h_\lambda(x) =  \frac{(1 - \alpha \Re\lambda + \beta)^2}{(x+\frac{\beta}{x})^2} + \frac{(\alpha \Im \lambda)^2}{(x - \frac{\beta}{x})^2}$, which is continuous and non-increasing. We show the result by double inclusion.
\begin{itemize}
    \item Let $\lambda \in \interior S(\alpha,\beta,\rho)$. As $\rho > 0$, $h_\lambda(\rho) < 1$ by \cref{eq: proof lemma interior momentum convergence regions}. As $\rho > \sqrt{|\beta|}$, by continuity of $h_\lambda$ at $\rho$, there exists $\rho > \rho' > \sqrt{|\beta|}$ such that $h_\lambda(\rho') < 1$. As $\rho' > \sqrt{|\beta|} \geq 0$, this implies that $\lambda \in S(\alpha, \beta,\rho')$.
    \item Let $\rho' > 0$ such that $|\beta| < \rho' < \rho$ and take $\lambda \in S(\alpha,\beta,\rho')$. By \cref{lemma: momentum convergence region}, as $\rho' > 0$, this implies that $h_\lambda(\rho') \leq 1$. Note that if both $\Im \lambda = 0$ and $1 - \alpha \Re \lambda + \beta = 0$, $\lambda \in \interior S(\alpha, \beta, \rho)$ as $\rho > 0$. Otherwise, if at least one of them is non-zero, this means that $h_\lambda$ is actually decreasing on $]\sqrt{|\beta|}, +\infty[$. Hence, $h_\lambda(\rho) < h_\lambda(\rho') \leq 1$ and so $\lambda \in \interior S(\alpha, \beta, \rho)$.
\end{itemize}
\end{proof}
\begin{lemma}\label{lemma: boundary of momentum convergence regions}
 If $0 < |\beta| \leq \rho^2$,
\begin{equation}\label{eq: lemma boundary momentum convergence regions}
\partial S(\alpha, \beta, \rho) = S(\alpha, \beta, \rho) \cap \{\lambda \in \C: \exists z \in \C,\, z^2 - (1 - \alpha \lambda + \beta) z + \beta = 0 \text{ and } |z| = \rho \}
\end{equation}
\end{lemma}
\begin{proof}
This is a direct consequence of \cref{lemma: interior momentum convergence regions} and the definition of $S(\alpha, \beta, \rho)$.
\end{proof}
\begin{lemma}\label{lemma: undefined set}
For $0 < \beta < \rho^2$,
\begin{equation}
\{\lambda \in \R: (1 - \alpha \lambda + \beta)^2 \leq 4\beta \} \subset \interior (S(\alpha, \beta, \lambda))
\end{equation}
For $0 < -\beta < \rho^2$,
\begin{equation}
\{\lambda \in \C: 1 - \alpha \Re \lambda +\beta = 0,\, (\alpha \Im \lambda)^2\leq 4|\beta| \} \subset \interior (S(\alpha, \beta, \lambda))
\end{equation}
\end{lemma}
\begin{proof}
First assume that $0 < \beta < \rho^2$.
Consider $\lambda \in \R$ such that  $(1 - \alpha \lambda + \beta)^2 \leq 4\beta$. Using the characterization of \cref{lemma: momentum convergence region}, we only need to show that $(1 - \alpha \lambda + \beta)^2 < \rho^2(1 + \tau)^2$ (as $-\rho^2 < \beta \implies \tau \neq -1$). But, from the definition of $\tau$ we get
\begin{equation}
\rho^2(1 + \tau)^2 - 4\beta = (\rho - \frac{\beta}{\rho})^2 > 0.
\end{equation}
Hence $4\beta < \rho^2(1 + \tau)^2$ and the result follows from the choice of $\lambda$.

The proof for the second point is similar.

\end{proof}
 We can now prove the proposition which was the target of this subsection. Note that the following proof does not encompass the case particular $\rho^2 = |\beta|$ in which the ellipse is degenerate. This falls into the case of segments, which is much simpler, see the aforementioned references (or \citet{nevanlinnaConvergenceIterationsLinear1993} for a didactic explanation).

\begin{prop}\label{prop: acf for momentum convergence region 2}
Assume $0 < |\beta| < \rho^2 < 1$ and $\alpha > 0$, then \begin{equation}
    \acf(S(\alpha, \beta, \rho)) = \rho
\end{equation}
\end{prop}
\begin{proof}%
We will build the Green function for $\bar\C \setminus S(\alpha, \beta, \rho)$ using \cref{lemma: joukowski inverse map}.

First, we show that if $\lambda \notin \interior S(\alpha, \beta, \rho)$, then $\frac{1 - \alpha \lambda + \beta}{2\sqrt{\beta}} \notin [-1, 1]$ where $\sqrt{\beta}$ is a square root (with positive real part) of $\beta$. Indeed, assume for the sake of contradiction that it is not the case, i.e. there exists $\lambda \notin \interior S(\alpha, \beta, \rho)$ such that $\frac{1 - \alpha \lambda + \beta}{2\sqrt{\beta}} \in [-1, 1]$. Assume first that $\beta > 0$. This implies that $\Im(1 - \alpha \lambda + \beta) = 0$ and so that $\lambda \in \R$ as $\alpha \neq 0$. Moreover, as $\beta > 0$, $\lambda$ satisfies $(1 - \alpha \lambda + \beta)^2 \leq 4\beta$. By \cref{lemma: undefined set}, $\lambda \in \interior S(\alpha, \beta, \rho)$ which is a contradiction. If $\beta < 0$, $\sqrt{\beta} = i\sqrt{|\beta|}$. This implies that $\Re(1 - \eta \lambda + \beta) = 0$. Moreover, $\lambda$ satisfies $(\Im(1 - \alpha \lambda + \beta))^2 \leq 4|\beta|$. We get a similar contradiction using \cref{lemma: undefined set}.

Take $\lambda \notin \interior S(\alpha, \beta, \rho)$. Then, as $\frac{1 - \alpha \lambda + \beta}{2\sqrt{\beta}} \notin [-1, 1]$, we can consider $\phi\left(\frac{1 - \alpha \lambda + \beta}{2\sqrt{\beta}}\right)$. By \cref{lemma: joukowski inverse map}, $\phi\left(\frac{1 - \alpha \lambda + \beta}{2\sqrt{\beta}}\right)$ is the unique solution of modulus (strictly) greater than one of
\begin{equation}
z^2 - 2\frac{1 - \alpha \lambda + \beta}{2\sqrt{\beta}}z + 1 = 0\,. 
\end{equation}
Hence $\sqrt{\beta}\phi\left(\frac{1 - \alpha \lambda + \beta}{2\sqrt{\beta}}\right)$ is the unique solution of modulus (strictly) greater than $\sqrt{|\beta|}$ of
\begin{align}
&\frac{z^2}{\beta} - 2\frac{1 - \alpha \lambda + \beta}{2\beta}z + 1 = 0\\ 
\iff &{z^2} - (1 - \eta \lambda + \beta)z + \beta = 0\,.\label{eq: eq defining ellipsis}
\end{align}
Let $z_1 = \sqrt{\beta}\phi\left(\frac{1 - \alpha \lambda + \beta}{2\sqrt{\beta}}\right)$ and let $z_2$ be the other root of \cref{eq: eq defining ellipsis}. Then $z_1z_2 = \beta$ and so $|z_1z_2| = |\beta|$. Hence, as $|z_1| > \sqrt{|\beta|}$, we have $|z_2| < \sqrt{|\beta|}$. Hence $\sqrt{\beta}\phi\left(\frac{1 - \alpha \lambda + \beta}{2\sqrt{\beta}}\right)$ is the solution of greatest magnitude of \cref{eq: eq defining ellipsis}. We will see that this quantity is very regular as a function of $\lambda$ outside $S(\alpha, \beta, \rho)$.
Define 
\begin{equation}
\chi:\left\{\begin{aligned}
\C\setminus \interior S(\alpha, \beta, \rho) & \longrightarrow \C\\
\lambda\phantom{\beta, \rho)} & \longmapsto \sqrt{\beta}\phi\left(\frac{1 - \alpha \lambda + \beta}{2\sqrt{\beta}}\right)\,.\\
\end{aligned}\right.
\end{equation}
We can now build our Green function using $\chi$. Define, 
\begin{equation}
g:\left\{\begin{aligned}
\C\setminus \interior S(\alpha, \beta, \rho) & \longrightarrow \R\\
\lambda\phantom{\beta, \rho)} & \longmapsto \log\frac{|\chi(\lambda)|}{\rho}\,.\\
\end{aligned}\right.
\end{equation}
Note that as $\phi$ is continuous on its domain of definition $\chi$ is continuous too. Moreover, as $\beta \neq 0$ and $\chi(\lambda)$ is a root of \cref{eq: eq defining ellipsis}, $\chi(\lambda) \neq 0$ for $\lambda \notin \interior S(\alpha, \beta, \rho)$. Hence $g$ is well-defined and continuous too on $\C\setminus \interior S(\alpha, \beta, \rho)$.
We now show that $g$ is the Green function of $G_\infty = \C\setminus S(\alpha, \beta, \gamma)$ according to \cref{definition: green function}.
\begin{enumerate}
    \item By \cref{lemma: joukowski inverse map}, $\phi$ is analytic and so is $\chi$ on the open set $\C \setminus S(\alpha,\beta,\rho)$. Moreover, as mentioned above, $\chi(\lambda) \neq 0$ for $\lambda \notin S(\alpha, \beta, \rho)$ Hence, $g$ is harmonic on $\C \setminus S(\alpha,\beta,\rho) = G_\infty$.
    \item When $\lambda \rightarrow \infty$, $\frac{1 - \alpha \lambda + \beta}{2\sqrt{\beta}} \rightarrow \infty$ too as $\alpha \neq 0$. Hence, by \cref{lemma: joukowski inverse map},
    \begin{align}
        g(\lambda) &= \log \frac{|\chi(\lambda)|}{\rho}\\
                   &= \log \left|\phi\left(\frac{1 - \alpha \lambda + \beta}{2\sqrt{\beta}}\right)\right| + \bigO(1)\\
                   &= \log |\lambda + \bigO(1)| + \bigO(1)\\
                   &= \log |\lambda| + \bigO(1)\,.\\
    \end{align}
    \item Let $\zeta \in \partial \left(\C \setminus S(\alpha, \beta, \gamma)\right) = \partial S(\alpha, \beta, \gamma)$. Note that $\chi$ is defined on $\C\setminus \interior S(\alpha, \beta, \rho)$ on so on $\partial S(\alpha, \beta, \gamma)$.  Then, by \cref{lemma: boundary of momentum convergence regions} and the definition of $\chi$, $|\chi(\zeta)| = \rho$. By continuity of $g$, $g(\lambda) \xrightarrow[\lambda \rightarrow \zeta]{} g(\zeta) = 0$.
Hence $g$ is the Green function for $G_\infty$ by \cref{definition: green function}. Moreover, by \cref{lemma: zero not in momentum convergence region}, $0 \in  G_\infty$. We can now apply \cref{thm: link asf green function} to get that $\acf(S(\alpha,\beta,\rho)) = \exp(-g(0))$. Finally, we compute $g(0)$. Recall that, as $0 \notin S(\alpha, \beta, \rho)$,  $\chi(0)$ is the root of greatest magnitude of 
\begin{equation}
z^2 - (1+\beta)z + \beta = 0\,.    
\end{equation}
The two roots of this equation are $\beta$ and 1. As $0 < |\beta| < 1$, $\chi(0) = 1$, so $g(0) = \log \frac{1}{\rho}$ and $\acf(S(\alpha,\beta,\rho)) = \rho$.
\end{enumerate}
\end{proof}

\newpage

\section{Synthetic Experiments}\label{section: app synthetic exp}
In this section we evaluate the accelerated methods we studied on synthetic bilinear games.

We consider bilinear games of the form,
\begin{equation*}
    \min_{x \in \R^m}\max_{y \in \R^m} (x - x^*)^\top A(y - y^*)\,.
\end{equation*}
$A$, $x^*$, $y^*$ and the initial points are chosen randomly. More precisely, each of their coefficients is drawn from a standard normal distribution and $A$ is normalized such that $\frac{\sigma_{max}(A)}{\sigma_{min}(A)} = 100$. The total dimension of the parameter space is $d=2m$.

We compare the accelerated methods we presented to methods which are proved to converge on such games: EG, Hamiltonian gradient descent (HGD) \citep{abernethyLastiterateConvergenceRates2019a}, the alternating gradient method with negative momentum \citep{gidelNegativeMomentumImproved2018b} and optimistic mirror descent (OMD) \citep{daskalakisTrainingGANsOptimism2017a}. 

\begin{figure}[h]
\begin{tikzpicture}
\begin{groupplot}[group style={group size= 2 by 2,horizontal sep=3cm, vertical sep=3cm},height=7cm,width=8cm]

\nextgroupplot[title={$d=100$}, ylabel=Distance to optimum, xlabel=Iterations,
    legend pos = south west,
    xmin = 0, xmax = 1000,
    ymin = 0.0000001, ymax = 100, ymode=log
    ]
    \pgfplotstableread[col sep = comma]{xp-100.csv}\data
    \addplot +[mark=none] table[x index = {0}, y index = {1}]{\data};
	\addplot +[mark=none] table[x index = {0}, y index = {2}]{\data};
	\addplot +[mark=none] table[x index = {0}, y index = {3}]{\data}; 
	\addplot +[mark=none] table[x index = {0}, y index = {4}]{\data};
	\addplot +[mark=none, dashed] table[x index = {0}, y index = {5}]{\data};
	\addplot +[mark=none, dashed] table[x index = {0}, y index = {6}]{\data}; 
	\addplot +[mark=none, dashed] table[x index = {0}, y index = {7}]{\data};    
    \legend{Extragradient, HGD, Neg. Mom., OMD, This work \S 5.2, This work \S 5.4, This work \S 6.2}
    
  \nextgroupplot[title={$d=500$}, ylabel=Distance to optimum,xlabel=Iterations, xmin = 0, xmax = 1000,
    ymin = 0.0000001, ymax = 100, ymode=log,
    ]
    \pgfplotstableread[col sep = comma]{xp-500.csv}\data
    \addplot +[mark=none] table[x index = {0}, y index = {1}]{\data};
	\addplot +[mark=none] table[x index = {0}, y index = {2}]{\data};
	\addplot +[mark=none] table[x index = {0}, y index = {3}]{\data}; 
	\addplot +[mark=none] table[x index = {0}, y index = {4}]{\data};
	\addplot +[mark=none, dashed] table[x index = {0}, y index = {5}]{\data};
	\addplot +[mark=none, dashed] table[x index = {0}, y index = {6}]{\data}; 
	\addplot +[mark=none, dashed] table[x index = {0}, y index = {7}]{\data};      

    \nextgroupplot[title={$d=1000$}, ylabel=Distance to optimum, xlabel=Iterations,
    legend pos = south east,
    xmin = 0, xmax = 1000,
    ymin = 0.0000001, ymax = 100, ymode=log,
    ]
    \pgfplotstableread[col sep = comma]{xp-1000.csv}\data
    \addplot +[mark=none] table[x index = {0}, y index = {1}]{\data};
	\addplot +[mark=none] table[x index = {0}, y index = {2}]{\data};
	\addplot +[mark=none] table[x index = {0}, y index = {3}]{\data}; 
	\addplot +[mark=none] table[x index = {0}, y index = {4}]{\data};
	\addplot +[mark=none, dashed] table[x index = {0}, y index = {5}]{\data};
	\addplot +[mark=none, dashed] table[x index = {0}, y index = {6}]{\data}; 
	\addplot +[mark=none, dashed] table[x index = {0}, y index = {7}]{\data};        
  \end{groupplot}
\end{tikzpicture}
\caption{Distance to the optimum as a function of the number of iterations of the methods.}
\end{figure}
\end{document}

%% file: param.tex
\usepackage[english]{babel}
\usepackage[utf8]{inputenc}
\usepackage{amsmath,amssymb,amsthm, amsfonts, mathtools}
\usepackage[colorinlistoftodos]{todonotes}
\usepackage{microtype}
\usepackage{graphicx}
\usepackage{subcaption}
\usepackage{accents}
\usepackage{booktabs}
\definecolor{mydarkblue}{rgb}{0,0.08,0.45}
\usepackage[colorlinks=true,
    linkcolor=mydarkblue,
    citecolor=mydarkblue,
    filecolor=mydarkblue,
    urlcolor=mydarkblue]{hyperref}
\usepackage{url}
\usepackage{nicefrac}
\usepackage{microtype}
\usepackage{array}
\usepackage{verbatim}
\newcolumntype{P}[1]{>{\centering\arraybackslash}p{#1}}

\newcommand{\R}{\mathbb{R}}
\newcommand{\C}{\mathbb{C}}
\newcommand{\bigO}{\mathcal{O}}
\newcommand{\bgO}[1]{\mathcal{O}\left(#1\right)}

\newcommand{\acf}{\underaccent{\bar}{\rho}}
\renewcommand{\P}{\mathcal{P}}

\DeclareMathOperator*{\argmax}{arg\,max}
\DeclareMathOperator*{\argmin}{arg\,min}
\DeclareMathOperator*{\Sp}{Sp}
\DeclareMathOperator*{\diag}{diag}

\DeclareMathOperator*{\Id}{Id}

\DeclareMathOperator*{\interior}{int}

\DeclarePairedDelimiterX{\inp}[2]{\langle}{\rangle}{#1, #2}
\usepackage{thmtools}
\usepackage{thm-restate}

\declaretheorem[name=Theorem]{thm}
\declaretheorem[name=Theorem, numbered=no]{thm*}
\declaretheorem[name=Proposition]{prop}
\declaretheorem[name=Lemma]{lemma}
\declaretheorem[name=Corollary]{cor}
\declaretheorem[name=Definition, style=definition]{definition}

\declaretheorem[name=Remark, style=remark]{rmk}

\usepackage{pgfplotstable} %
\usetikzlibrary{automata, positioning, arrows, shapes, fit, calc, intersections}
\usepgfplotslibrary{fillbetween}
\usepgfplotslibrary{statistics}
\usepgfplotslibrary{groupplots}

%% file: figure_geometric.tex
\newcommand{\colorone}{blue!100!purple}
\newcommand{\colortwo}{red!100!orange}
\newcommand{\colorthree}{green!100!yellow}

\begin{figure}[t]
    \centering
    \newcounter{t}
    \setcounter{t}{80}
    \newcounter{L}
    \setcounter{L}{10}
    \begin{tikzpicture}[scale=0.67]
        \begin{axis}[
            grid=major,
            %From https://tex.stackexchange.com/questions/61039/how-to-keep-a-11-scale-with-x-and-y-axis
            axis equal image,
            yticklabel={
            $\pgfmathprintnumber{\tick}i$
            },
        	xmin=0,   xmax={\value{L} + 2},
        	ymin={-\value{L} },   ymax={\value{L}},
        	]
          \draw [domain=-180:180, draw=\colortwo, fill=\colortwo!50!white, thick, fill opacity=0.5, samples=65] plot (axis cs: {\value{L}*(1 + cos(\value{t}))/2 + (cos(\x) - sin(abs(2*\x))^2/4 )* \value{L}*(1 - cos(\value{t}))/2}, {sin(\x) * \value{L}*(1 - cos(\value{t}))});
          \node [circle, anchor=south west] (L) at (axis cs: {\value{L}/2}, 0) {${\color{red}K}$};
      \end{axis}
    \end{tikzpicture}
    \begin{minipage}[c]{.01\textwidth}
        \vspace{-5cm}
        $\overset{\varphi}{
        \longrightarrow}$
    \end{minipage}
        \setcounter{t}{80}
    \setcounter{L}{10}
    \newcounter{mu1}
    \setcounter{mu1}{1}
    \newcounter{mu2}
    \setcounter{mu2}{2}
    \newcounter{L2}
    \setcounter{L2}{6}
        \begin{tikzpicture}[scale=0.67]
                \hspace{-4mm}
        \begin{axis}[
            grid=major,
            axis equal image,
            yticklabel={
            $\pgfmathprintnumber{\tick}i$
            },
        	xmin=0,   xmax={\value{L} + 1},
        	ymin={-\value{L} + 3},   ymax={\value{L} - 3},
        	]
          \draw [domain=-180:180, thick, draw=\colorone, fill=\colorone!50!white, fill opacity=0.25, samples=65] plot (axis cs: {(\value{L} + \value{mu1})/2 + cos(\x) * (\value{L} - \value{mu1})/2}, {sin(\x) * sqrt(\value{mu1}*\value{L})});
          \draw [domain=-180:180, draw=\colortwo, fill=\colortwo!50!white, thick, fill opacity=0.45, samples=65] plot (axis cs: {
          \value{L}*(1 + cos(\value{t}))/2 + (cos(\x) - sin(abs(2*\x))^2/4 )* \value{L}*(1 - cos(\value{t}))/2
          - 1/\value{L}/3 * (\value{L}*(1 + cos(\value{t}))/2 + (cos(\x) - sin(abs(2*\x))^2/4 )* \value{L}*(1 - cos(\value{t}))/2)^2 
          + 1/\value{L}/3 * (sin(\x) * \value{L}*(1 - cos(\value{t})))^2
          }, 
          {
          sin(\x) *\value{L}*(1 - cos(\value{t}))/2
          -1/\value{L} /3 * (sin(\x) * \value{L}*(1 - cos(\value{t})))*(\value{L}*(1 + cos(\value{t}))/2 + (cos(\x) - sin(abs(2*\x))^2/4 )* \value{L}*(1 - cos(\value{t}))/2) 
          });
          \draw [domain=-180:180, thick, draw=\colorthree, fill=\colorthree!50!white, fill opacity=0.45, samples=65] plot (axis cs: {(\value{L2} + 1/4 + \value{mu2})/2 + cos(\x) * (\value{L2} + 8/10 - \value{mu2})/2}, {sin(\x) * sqrt(\value{mu2}*(\value{L2}-2 - 1/5)});
          \node [circle, anchor=south east] (L) at (axis cs: {\value{L}/2}, 0) {${\color{red}\varphi(K)}$};
      \end{axis}
    \end{tikzpicture}
    \caption{\small
    Transformation of the spectral shape {\color{red}$K$} ({\color{red}in red} from left to right) by the extragradient operator $\varphi: \lambda \mapsto \lambda(1-\eta\lambda)$. Any ellipse (e.g. in {\color{blue}blue}) that contains the transformed red shape {\color{red}$\varphi(K)$} provides a upper convergence bound using extragradient with Polyak momentum (with step-size and momentum that depends on the {\color{blue}ellipse} parameters). Any ellipse included in it (e.g. in {\color{ao}green}) provides a lower bound. See \S\ref{sec:geometric_interpretation}.
\vspace{-2ex}}\label{figure:geometric_example}
    \vspace{-1.0mm}
\end{figure}